%% file: anonymous-submission-latex-2024.tex
\documentclass[letterpaper]{article} 
\usepackage{aaai24}
\usepackage{times}  
\usepackage{helvet}  
\usepackage{courier}  
\usepackage[hyphens]{url}  
\usepackage{graphicx} 
\usepackage{natbib}  
\usepackage{caption} 
\usepackage{algorithm}
\usepackage{algorithmic}
\input{commands.tex}
\usepackage{newfloat}
\usepackage{listings}
\DeclareCaptionStyle{ruled}{labelfont=normalfont,labelsep=colon,strut=off} 
\floatstyle{ruled}
\newfloat{listing}{tb}{lst}{}
\floatname{listing}{Listing}
\title{Meta-Inverse Reinforcement Learning for Mean Field Games via Probabilistic Context Variables}
\author{
   Yang Chen\textsuperscript{\rm 1,2}, Xiao Lin\textsuperscript{\rm 3}, Bo Yan\textsuperscript{\rm 3}, Libo Zhang\textsuperscript{\rm 2}, Jiamou Liu\textsuperscript{\rm 2}, Neset \"Ozkan Tan\textsuperscript{\rm 1,2}, Michael Witbrock\textsuperscript{\rm 1,2}
}
\affiliations{
    \textsuperscript{\rm 1} NAOInstitute, University of Auckland, New Zealand\\
    \textsuperscript{\rm 2} School of Computer Science, University of Auckland, New Zealand\\
    \textsuperscript{\rm 3} School of Computer Science, Beijing Institute of Technology, Beijing, China\\

    \{yang.chen,jiamou.liu,neset.tan,m.witbrock\}@auckland.ac.nz, lzha797@aucklanduni.ac.nz, linxiao\_zj@foxmail.com, yanbo@bit.edu.cn
}

\begin{document}

\maketitle

\begin{abstract}
Designing suitable reward functions for numerous interacting intelligent agents is challenging in real-world applications. Inverse reinforcement learning (IRL) in mean field games (MFGs) offers a practical framework to infer reward functions from expert demonstrations. While promising, 
the assumption of agent homogeneity limits the capability of existing methods to handle demonstrations with heterogeneous and unknown objectives, which are common in practice. To this end, we propose a deep latent variable MFG model and an associated IRL method. Critically, our method can infer rewards from different yet structurally similar tasks without prior knowledge about underlying contexts or modifying the MFG model itself. 
 Our experiments, conducted on simulated scenarios and a real-world spatial taxi-ride pricing problem, demonstrate the superiority of our approach over state-of-the-art IRL methods in MFGs. 
\end{abstract}

\input{intro.tex}

\input{related.tex}

\input{pre.tex}

\input{methods.tex}
\input{experiments.tex}

\input{conclusions.tex}

\bibliography{bib}
\input{appendix.tex}
\end{document}

%% file: commands.tex
\usepackage{multicol}
\usepackage{multirow}
\usepackage{booktabs}
\usepackage{appendix}
\usepackage{comment}
\usepackage{amssymb}
\usepackage{bm}
\usepackage{dsfont}
\usepackage{amsthm}
\usepackage{algorithm}
\usepackage{algorithmic}
\usepackage{amsmath}
\newtheorem{proposition}{Proposition}
\newtheorem{lemma}{Lemma}

\newcommand{\Amc}[0]{{{\mathcal{A}}}}

\newcommand{\Dmc}[0]{{{\mathcal{D}}}}

\newcommand{\Kmc}[0]{{{\mathcal{K}}}}
\newcommand{\Lmc}[0]{{{\mathcal{L}}}}
\newcommand{\Mmc}[0]{{{\mathcal{M}}}}

\newcommand{\Pmc}[0]{{{\mathcal{P}}}}

\newcommand{\Smc}[0]{{{\mathcal{S}}}}

\newcommand{\piv}[0]{{\bm{\pi}}}
\newcommand{\muv}[0]{{\bm{\mu}}}

\newcommand{\Ebb}{\mathbb{E}}
\newcommand{\Rbb}{\mathbb{R}}

\newcommand{\KL}{\mathrm{KL}}



%% file: intro.tex
\section{Introduction}
Understanding incentives among interacting agents in real-world decision-making and control tasks is a fundamental challenge in multi-agent systems. Inverse reinforcement learning (IRL) \citep{ng1999policy,ng2000algorithms} addresses this issue by inferring reward functions from expert demonstrations, offering a succinct representation of tasks. IRL in multi-agent systems serves two primary purposes. First, it aids in comprehending and predicting the objectives of interacting agents, such as determining the destinations of autonomous vehicles \citep{you2019advanced}. Second, it enables the design of agent environments with known reward signals to guide their behaviour as desired, akin to mechanism design \cite{fu2021evaluating}. 

A prominent challenge with IRL is the ``curse of the agent number'', {\em i.e.,} an increase in agent number leads to exponential interaction complexities, resulting in impractical time and memory expenses. Fortunately, recent advancements in IRL within mean field games (MFGs) \citep{yang2018learning,chen2022individual,chen2023adversarial} offer a solution. By leveraging mean-field approximation, these methods simplify interactions among many agents to only two agents (individual-population), mitigating the computational burden. These MFG-based approaches have demonstrated compelling outcomes across diverse large-scale multi-agent tasks. Applications include modelling and predicting social media population behaviour \citep{yang2018learning}, product pricing in expansive markets, virus propagation modelling and explaining emerging social norms \citep{chen2022individual,chen2023adversarial}.


While appealing, these existing methods typically rely on the theoretically powerful yet practically unideal assumption in MFGs: all agents are homogeneous, {\em i.e.,} they are identical in the reward function, state-action space and dynamics. Real-world scenarios, however, frequently involve demonstrated behaviour with distinct and unknown rewards. For example, in spatial pricing for taxi rides, drivers may have different preferences based on factors such as distance, origins, and passenger destinations \cite{ata2019spatial}.
To address this issue, efforts have been made to generalise mean-field approximation by introducing additional type variables for each agent to differentiate reward functions \citep{subramanian2019reinforcement,ganapathi2020multi,ghosh2020model}. However, this approach requires prior knowledge of these variables, making it unsuitable for handling agents with unknown types or contexts, as observed in taxi trajectories.


Additionally, incorporating types into mean-field approximation makes the model more complex, necessitating the reevaluation of important theoretical properties, such as the existence and uniqueness of an equilibrium \citep{ghosh2020model}, to ensure a well-defined corresponding IRL problem. Although such generalisations are theoretically valued, from a machine learning viewpoint, simpler models with fewer theoretical restrictions are often preferred. Given the considerations above, the question arises: {\em can we enable IRL to handle numerous agents with {\bf\em unknown} reward functions {\bf\em without altering} the mean-field approximation?}

On the other hand, an emerging branch of IRL called meta-IRL \cite{seyed2019smile,yu2019meta,xu2019learning} combines IRL and meta-learning to address similar reward structures among demonstrations from different tasks. It introduces a latent probabilistic context variable that influences the reward function and performs context-conditioned reward inference without prior knowledge of the contexts. Although currently limited only to single-agent scenarios, this approach inspires the idea of assigning probabilistic context variables (representing types) externally to a family of MFGs instead of internally to agents within an MFG. By doing so, each MFG in this family retains its original theoretical properties, enabling the development of a meta-IRL method on top of it.

\begin{figure}
	\centering
	\includegraphics[width=.48\textwidth]{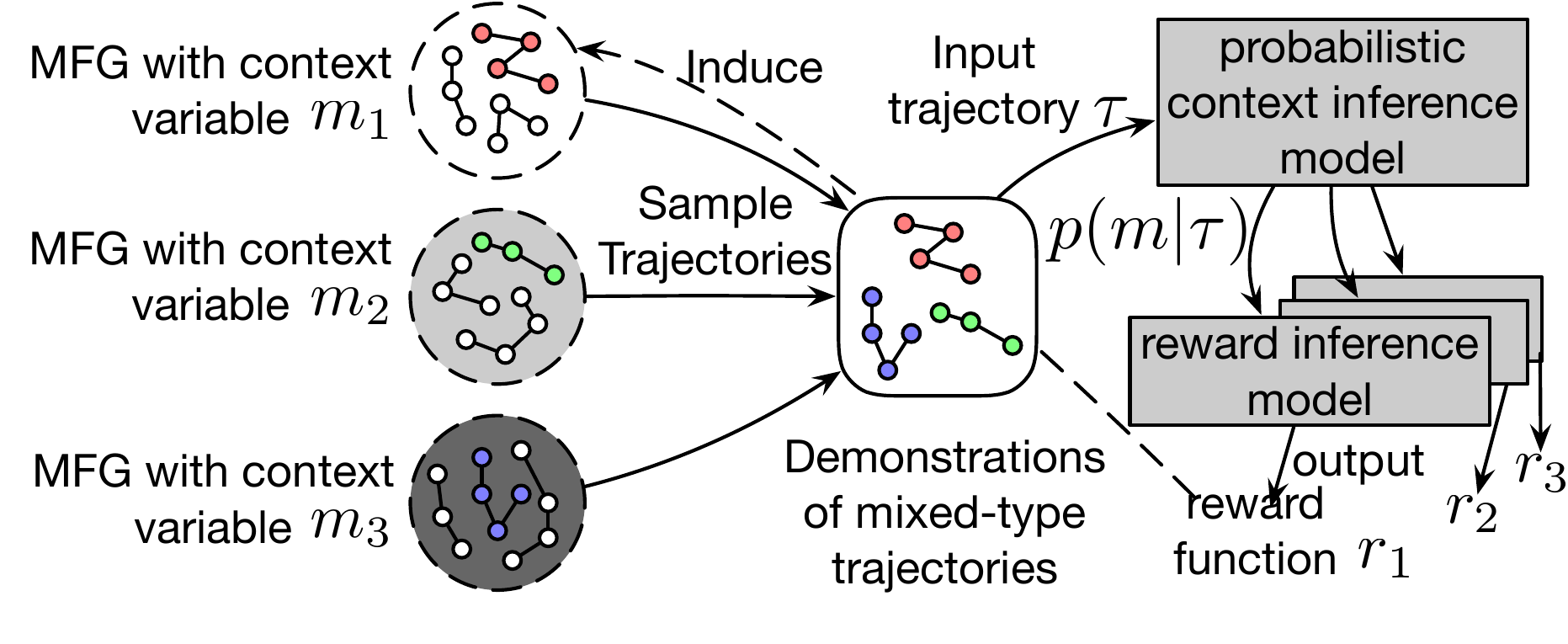}
	\caption{An overview of the mechanism of PEMMFIRL. The context variable is unknown to the framework.}\label{fig:overview}
\end{figure}

This paper implements this idea into a novel IRL framework called Probabilistic Embedding for Meta Mean Field IRL (PEMMFIRL) that answers the question we raised above. It integrates meta-IRL, mean-field approximation, and latent variable models into a unified framework, as illustrated in Fig.~\ref{fig:overview}.
Our contributions are threefold: 
\begin{enumerate}
	\item We extend the notion of MFG by introducing a probabilistic contextual variable, significantly enhancing its ability to handle heterogeneous agents without introducing additional assumptions to mean-field approximation.
	\item We develop PEMMFIRL, an associated IRL framework for this generalised MFG model, capable of inferring reward functions from demonstrations with different context variables. Importantly, PEMMFIRL can infer rewards from structurally similar tasks without prior knowledge about underlying contexts.
	\item Experimental results on simulated tasks and a real-world spatial taxi-ride pricing problem demonstrate the effectiveness of our approach, achieving an outstanding increase in drivers' average profit when applied to real-world passenger demand data.
\end{enumerate}

%% file: related.tex
\section{Related Work}\label{sec:related}

Our work is closely related to the literature on MFGs, initially introduced in the continuous setting by \citep{huang2006large} and \citep{lasry2007mean}. Later, MFGs were extended to the discrete model by \citep{gomes2010discrete}, which is commonly adopted in agent learning. Recently, learning MFGs has garnered significant attention \citep{cardaliaguet2017learning}, with existing methods relying on reinforcement learning techniques \citep{yang2018mean,guo2019learning,subramanian2019reinforcement,cui2021approximately}.
Our method stands out by directly recovering the reward function from observed behaviour, eliminating the need for manual reward design.

IRL was introduced by \citep{ng2000algorithms} for single-agent settings. Early IRL methods based on margin optimization \citep{ratliff2006maximum} were ill-defined. To resolve this, Maximum Entropy IRL (MaxEnt IRL) was proposed by \citep{ziebart2008maximum,ziebart2010modeling}. 
However, MaxEnt IRL is limited to small, discrete problems due to its iterative nature in reward function tuning. To extend MaxEnt IRL to high-dimensional or continuous domains, \citep{fu2018learning} introduced Adversarial IRL (AIRL). It utilises a sampling-based approximation by relating MaxEnt IRL to generative adversarial networks \citep{goodfellow2014generative}, enabling effective reward tuning for complex scenarios. Some recent work extends AIRL to the multi-agent setting \citep{yu2019multi,fu2021evaluating} and the mean-field setting \cite{yang2018learning,chen2023adversarial}. However, they are limited to either the countable-agent or the homogeneous many-agent cases.

We build our approach on the problem of meta-learning, also known as ``learning to learn'' \cite{thrun2012learning}, which aims to train models to adapt quickly to new tasks. Many methods have been proposed including the memory-based methods \citep{duan2016rl,santoro2016meta,wang2016learning,mishra2017meta}, methods that learn an optimiser and/or initialisation \citep{andrychowicz2016learning,ravi2016optimization,finn2017model,sun2018learning}. Our method resembles \citep{rakelly2019efficient} that differentiates different tasks through a deep latent variable. Notably, meta-IRL \citep{xu2019learning,yu2019meta} incorporates meta-learning into IRL in order to rapidly adapt to new tasks or domains. While these methods show promise in both tabular and high-dimensional/continuous tasks, they can only deal with a handful of agents. Unlike these prior work, with the help of MFG, our approach is capable of handling super large-scale tasks. In this sense, our work can be seen as a significant extension of IRL to address many-agent problems with multi-task demonstrations.

%% file: pre.tex
\section{Preliminaries}\label{sec:pre}

\subsection{Mean Field Games}
Consider a population of homogeneous agents sharing the same finite (local) state space $\Smc$ and action space $\Amc$. Let $\{1,2,\ldots, N\}$ be an enumeration of agents. MFGs reduce the all-agent interactions to two-party interactions between a single representative agent and the population, wherein the population is embodied by an empirical distribution of states of the system, called a {\em mean field}, given by
\begin{equation*}
	\mu \in \Delta(\Smc) \text{ such that } \mu(s) \triangleq \lim_{N \to \infty} \frac{1}{N} \sum_{i=1}^N \mathds{1}_{\{s_i = s\}},
\end{equation*}
where 
$\mathds{1}$ denotes the indicator function, {\em i.e.,} $\mathds{1}_{x} = 1$ if $x$ is true and $0$ otherwise. The reward function $r:\Smc\times \Amc \times \Delta(\Smc) \to \Rbb$ and state transition dynamics $P\colon \Smc \times \Amc \times \Delta(\Smc) \times \Smc \to [0,1]$ of the representative agent thus depend on the current state, action and the additional mean field. Let $T > 0$ be a finite horizon and the initial mean field $\mu^0$ be given, the sequence of mean fields $\muv = \mu^0, \mu^1,\ldots, \mu^T$  is called a {\em mean-field flow}. Likewise, a {\em policy follow} $\piv = \pi^0, \pi^1, \dots, \pi^t$ determines the agent's strategy, where $\pi^t:\Smc \to \Delta(\Amc)$ 
maps from states to a distribution over actions. 

A mean-field flow $\muv$ is said to be {\em consistent} with a policy flow $\piv$ if for all $t< T$, $\mu^{t+1}$ matches an individual's state marginal distribution when executing $\pi^t$. This can be formally written by the discrete-time McKean-Vlasov (MKV) equation \citep{carmona2013control}:
\begin{equation}\label{eq:MKV}
	\mu^{t+1}(s') = \sum_{s \in \Smc} \mu^t(s) \sum_{a \in \Amc} \pi^t(a \vert s)\; P(s' \vert s, a, \mu^t).
\end{equation}
Let $\tau = s^0, a^0, \ldots, s^T, a^T$ denote a state-action trajectory of an individual. The trajectory distribution induced by a pair of mean field flow and policy flow can be written as:
\begin{equation*}\label{eq:traj_dist_p}
	p_{\muv, \piv}(\tau) = \mu^0(s^0) \prod_{t=0}^T \pi^t(a^t \vert s^t) P(s^{t+1} \vert s^t,a^t,\mu^t).
\end{equation*}
In particular, if $\muv$ is consistent with $\piv$, we can rewrite $p_{\muv, \piv}(\tau)$ as $p_{\muv, \piv}(\tau) = \prod_{t=0}^T \mu^t(s^t)\pi^t(a^t \vert s^t).$

A policy flow $\piv$ is said to be {\em optimal} to a given $\muv$ if it maximises the expected return $\Ebb_{\tau \sim p_{\muv, \piv}(\tau)}[\sum_{t=0}^T r(s^t, a^t, \mu^t)]$.\footnote{The reward at the last step is zero \citep{elie2020convergence}.} While, an optimal policy flow may not be unique. Entropy-regularised MFG \citep{cui2021approximately} resolves this ambiguity by augmenting the reward with the policy entropy $\Ebb_{a \sim \pi }[-\log \pi(a \vert s)]$, resulting in the following objective:
\begin{equation}\label{eq:entropy_obj}
	\max_{\piv} \Ebb_{\tau \sim p_{\muv, \piv}(\tau)}\left[ \sum_{t=0}^T r(s^t, a^t, \mu^t) -  \log \pi^t(a^t \vert s^t) \right].
\end{equation}
The solution concept is called the {\em entropy-regularised mean field Nash equilibrium} (ERMFNE) which is a pair of mean field flow and policy flow $(\muv, \piv)$ such that $\piv$ is optimal to $\muv$ and, in turn, $\muv$ is consistent with $\piv$, {\em i.e.,} it fulfils the condition in Eq.~\eqref{eq:MKV} and maximises the objective in Eq.~\eqref{eq:entropy_obj}.
Shown by \citep[Theorem~3]{cui2021approximately},  an ERMFNE exists uniquely under certain conditions.

\subsection{Inverse Reinforcement Learning for MFGs}\label{sec:MFIRL}

Suppose we have no access to the reward function $r(s,a,\mu)$ but have a set of expert demonstrated trajectories $\Dmc = \{\tau_j\}_{j=1}^M$ sampled from an ERMFNE $(\muv_E, \piv_E)$ via $s^0 \sim \mu^0, a^t \sim \pi_E^t(a \vert s^t), s^{t+1} \sim P(s \vert s^t, a^t, \mu_E^t)$. Mean field Adversarial IRL (MF-AIRL) \citep{chen2023adversarial} aims to recover the underlying reward function from demonstrations, which can be interpreted as the following optimisation problem:
\begin{equation}\label{eq:MFIRL}
	\min_\omega D_{\mathrm{KL}}\left(p_{\muv_E, \piv_E}(\tau) \parallel p_\omega(\tau)\right)
\end{equation}
\begin{equation*}\label{eq:traj_dist}
\begin{aligned}
p_\omega(\tau) = \frac{1}{Z(\omega)}\left[ \prod_{t=0}^{T} \mu_\omega^t(s^t) \right] \cdot  \exp\left( \sum_{t=0}^T r_\omega\left(s^t, a^t, \mu_\omega^t\right) \right)
\end{aligned}
\end{equation*}
Here, $r_\omega$ is the $\omega$-parameterised reward function,  $p_\omega$ denotes the probability that a trajectory is generated under the $r_\omega$-induced ERMFNE denoted by $(\muv_\omega, \piv_\omega)$. The summation $Z(\omega)$ denotes the partition function, {\em i.e.,} the sum over all trajectories. Directly optimising the objective above is intractable as we have no access to the analytical form of $\mu_\omega^t$, which is a result of the entanglement between the mean field and the policy in MFGs. Fortunately, we can bypass such entanglement by establishing an unbiased estimate of $\mu^t_E$:
\begin{equation}\label{eq:empirical-mf-mfirl}
	\hat{\mu}^t_E(s) = \frac{1}{M} \sum_{j=1}^M \mathds{1}_{s_j^t = s}.
\end{equation}  
Proven in \citep[Theorem~2]{chen2023adversarial}, with $\mu_\omega^t$ being substituted with $\hat{\mu}^t_E$, the solution to the above optimisation problem approaches the optimal reward parameter if the number of demonstrated trajectories is sufficiently large.

However, computing $Z_\omega$ is intractable if the state-action space is large. To address this issue, MF-AIRL takes the mechanism of adversarial IRL \citep{fu2018learning} 
and reframes Eq.~\eqref{eq:MFIRL} as 
optimising a  {\em generative adversarial network} \citep{goodfellow2014generative}. It uses a discriminator $D_\omega$ (a binary classifier) and a sequence of {\em adaptive samplers} $\piv_\theta$ (a policy flow) whose update is equivalent to improving a sampling-based approximation to $Z_\omega$. Particularly, $D_\omega$ takes the form of
	$D_\omega(s,a,\hat{\mu}_E^t) = \exp\left( f_\omega(s,a, \hat{\mu}^t_E)\right) / (\exp\left( f_\omega(s,a,\hat{\mu}^t_E)\right) + \pi_\theta^t(a \vert s)),$
where $f_\omega$ serves as the parameterised reward function. The update of $D_\omega$ is interleaved with the update of $\piv_\theta$: $D_\omega$ is trained to update the reward function by distinguishing between the trajectories sampled from the expert and the adaptive samplers, {\em i.e.,} to maximise $\Ebb_{\muv_E,\piv_E}[\log D_\omega] + \Ebb_{\piv_\theta}[\log (1 - D_\omega)]$; while $\piv_\theta$ is trained to maximise the entropy-augmented cumulative rewards $\Ebb_{\pi_\theta} \left[\log D_\omega - \log \left(1 - D_\omega \right) \right] = \Ebb_{\piv_\theta} \left[ f_\omega(s^t,a^t,\hat{\mu}^t_E) - \log \pi_\theta^t(a^t \vert s^t)\right].$ 
Under certain conditions, $f_\omega$ will recover the underlying reward function $r(s,a,\mu)$ \citep{chen2023adversarial}.

%% file: methods.tex
\section{Meta-Mean Field IRL with Probabilistic Context Variables}
\subsection{MFGs with Probabilistic Context Variables}
We extend the concept of MFGs by introducing a probabilistic contextual variable $m \in \Mmc$, which follows a prior distribution $p(m)$. Here, $\Mmc$ represents a discrete value space. MFGs with different values of $m$ can be viewed as multiple large-scale multi-agent tasks with a shared structure. A real-world example is taxi hailing, where each taxi is associated with a specific $m$ representing the driver's trip preferences.
As a result, the policy, mean field, and reward function now have dependencies on $m$, denoted as $\pi^t:\Smc \times \Mmc \to \Delta(\Amc)$, $\mu^t: \Mmc \to \Delta (\Smc)$, and $r:\Smc \times \Amc \times \Delta(\Smc) \times \Mmc \to \Rbb$, respectively. We assume independence among the state-action space, transition function \citep{finn2017model, rakelly2019efficient}, and initial mean field w.r.t. $m$, which leads to the following expression for the trajectory distribution conditioned on the additional context variable: $p_{\muv,\piv}(\tau \vert m) =$
\begin{equation*}
	\mu^0(s^0) \prod_{t=0}^T \pi(a^t \vert s^t, m) P(s^{t+1} \vert s^t,a^t,\mu^t(\cdot \vert m)).
\end{equation*}

Given the insights above and following the spirit of entropy-regularised MFG, our objective is to maximise the expected return over the additional probabilistic contextual variable while maintaining the consistency between the mean field flow and policy flow. This can be formulated as the following constrained optimisation problem:
\begin{equation}\label{eq:new_reward_obj}
	\begin{aligned}
		&\max_{\muv, \piv} \Ebb_{m,\tau} \left[ \sum_{t=0}^T  r(s^t, a^t, \mu^t,m)  - \log \pi^t(a^t \vert s^t,m) \right]\\
		 & \qquad\qquad \text{\em where } m \sim p(m), \tau \sim p_{\muv,\piv}(\tau \vert m) \\
		& \text{s.t. } \muv(s \vert m) \text{\em is consistent with } \piv(a \vert s, m), \forall m \in \Mmc.
	\end{aligned}
\end{equation}
Intuitively, for all $m\in\Mmc$, $\left(\muv(s \vert m), \piv(a \vert s, m)\right)$ constitutes the ERMFNE of the MFG specified by $m$.

\subsection{Problem Setup}\label{sec:setup}
We now introduce the problem of meta-mean field IRL (meta-MFIRL) with multi-task demonstrations. Let us consider the existence of a ground-truth reward function $r(s,a,\mu,m)$, a prior distribution $p(m)$, and a pair of $m$-conditioned mean field flow and policy flow $(\muv_E(\cdot \vert m), \piv_E(\cdot \vert m))$ that solves the constrained optimisation problem in Eq.~\eqref{eq:new_reward_obj}. Given a collection of demonstrated trajectories that are {\em i.i.d.} samples drawn from the resulting marginal distribution $p_{\muv_E, \piv_E}(\tau) = \sum_{m \in \Mmc} p(m)p_{\muv_E, \piv_E}(\tau \vert m)$, our objective is to meta-learn an inference model $q(m \vert \tau)$ and a reward function $f(s,a,\mu,m)$. The aim is to ensure that when presented with a new trajectory $\tau_E$ generated by sampling $m' \sim p(m)$ and $\tau_E \sim p_{\muv_E, \piv_E}(\tau \vert m')$, and with $\hat{m}$ being inferred as $\hat{m} \sim q(m \vert \tau_E)$, both $r(s,a,\mu,\hat{m})$ and $f(s,a,\mu,\hat{m})$ yield the same solution to the problem in Eq.~ \eqref{eq:new_reward_obj}. To illustrate dependencies between variables, we depict the graphic model underlying the meta-MFIRL problem in Fig.~\ref{fig:graphic_model}.

It is important to note that we assume no access to the prior task distribution $p(m)$, the underlying $m$ value for each trajectory, nor the transition dynamics $P(s' \vert s, a, \mu)$. Additionally, we suppose that the entire supervision comes solely from the demonstrated data, meaning that requesting additional demonstrations is not allowed.

\begin{figure}
	\centering
	\includegraphics[width=.4\textwidth]{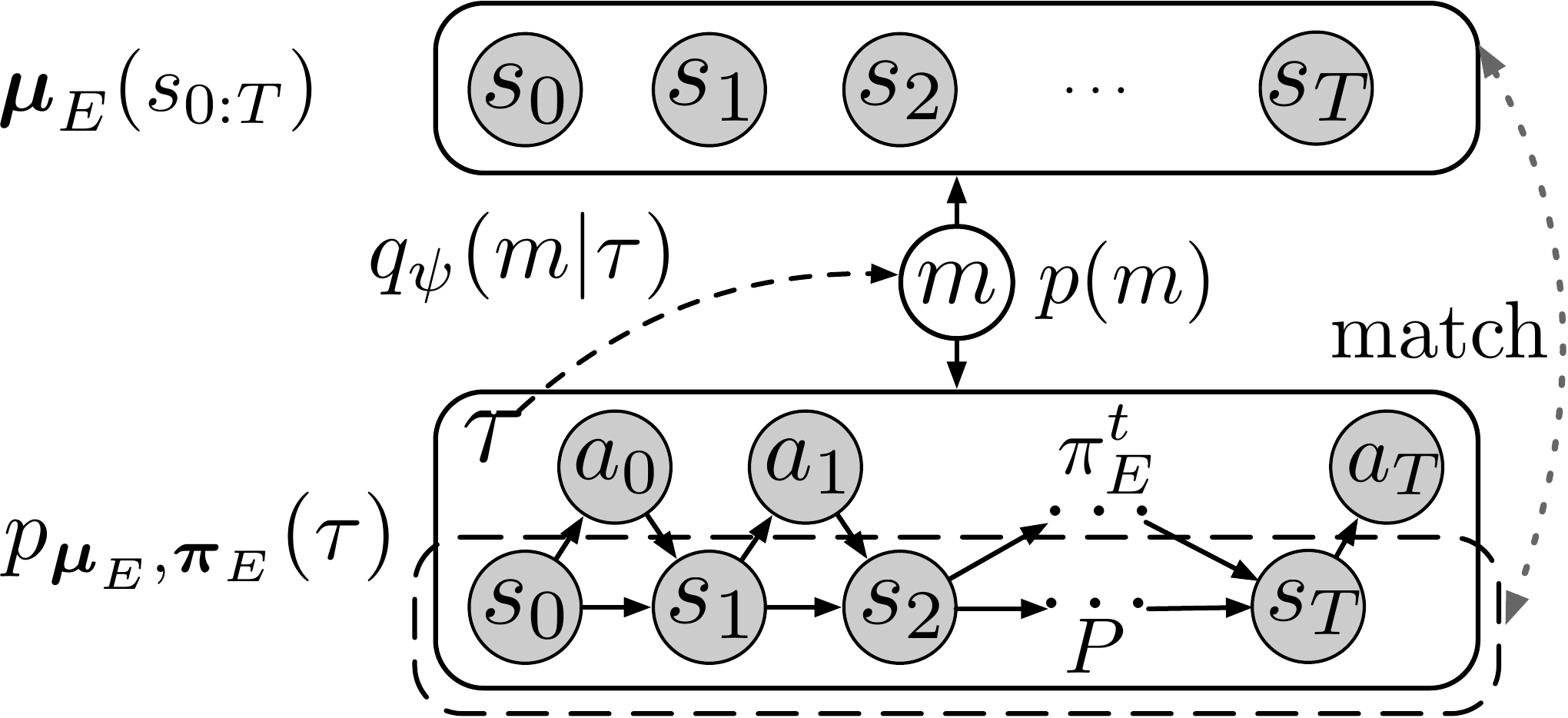}
	\caption{The graphic model of the meta-MFIRL problem. Note that the mean field (population's state density) matches an individual's state marginal when executing $\piv_E$ under $\muv_E$.}\label{fig:graphic_model}
\end{figure}

\subsection{Mutual Information Regularisation over Context Variables and the Reward Function}

Under MF-AIRL, we use a $\psi$-parameterised context variable inference model $q_\psi(m \vert \tau)$, an $\omega$-parameterised reward function $f_\omega(s,a,\mu,m)$.  With $m$ being inferred by $q_\psi(m \vert \tau)$, let $(\muv_\omega(\cdot \vert m), \piv_\omega(\cdot \vert s,m))$ denote the ERMFNE induced by $f_\omega$. The trajectory distribution under $(\muv_\omega, \piv_\omega)$ is given by:
\begin{equation}\label{eq:trajectory_theta_given_m}
	p_\omega(\tau \vert m) = \frac{1}{Z(\omega)}\left[ \prod_{t=0}^T \mu_\omega^t(s^t) \right] \cdot e^{\sum_{t=0}^T f_\omega\left(s^t, a^t, \mu_\omega^t, m \right)},
\end{equation}
where $\muv_\omega$ denotes the induced mean field flow consistent with $\piv_\omega$, and $Z(\omega)$ is the partition function.


Directly applying MF-AIRL without further constraints on the context variable $m$ would result in the pre-inferred $m$ by $q_\psi$ being treated as a constant in each component of MF-AIRL. As a result, maximising the likelihood with respect to the conditional distribution defined in Eq.~\eqref{eq:trajectory_theta_given_m} would lead to $m$ being ignored, making the learned reward function $f_\omega$ independent of $m$. Consequently, the learned reward function would be indistinguishable among different tasks. To establish a connection between the reward function and the context variable, it is essential to consider the mutual information between $m$ and the trajectories $\tau$ sampled from the reward-induced distribution. This measure quantifies the amount of information gained about the context variable by observing the trajectories, serving as an ideal measure, as done in \citep{zhao2018information,yu2019meta} from the perspective of information theory.
The mutual information between $m$ and $\tau$ under the joint distribution $p_\omega(m,\tau) = p(m) p_\omega(\tau \vert m)$ is given by:
\begin{equation*}\label{eq:mutual_information}
	I_{p_\omega}(m ; \tau) = \Ebb_{m\sim p(m), \tau \sim p_\omega(\tau \vert m)} \left[ \log p_\omega(m \vert \tau) - \log p(m) \right],
\end{equation*}                                                      
where $p_\omega(m \vert \tau)$ is the posterior distribution corresponding to the conditional distribution $p_\omega(\tau \vert m)$ as defined in Eq.~\eqref{eq:trajectory_theta_given_m}.

Optimising the mutual information is intractable as we have no access to the prior distribution $p(m)$, the posterior  $p(m \vert \tau)$, nor the $\omega$-induced conditional mean field $\mu_\omega^t(\cdot \vert m)$. To address this, we employ several approximations based on sampling and estimation. First, we replace $\mu_\omega^t(\cdot \vert m)$ with an empirical value estimated from demonstrations using $q_\psi(m \vert \tau)$, shifting its dependency from $\omega$ to $\psi$. Second, we use $q_\psi(m \vert \tau)$ as a variational approximation to $p_\omega(m \vert \tau)$, allowing for approximate sampling from $p(m)$. These approximations will be further explained in the next section. Now, our final goal in meta-MFIRL is to learn a reward function and an inference model for the task type. Formally, this can be interpreted as the pursuit of the following two desiderata:
\begin{enumerate}
	\item {\bf Reward desideratum.} Match conditional distributions: $$\Ebb_{m\sim p(m)}[D_{\KL}\left(p_{\muv_E, \piv_E}(\tau \vert m)|| p_\omega(\tau \vert m)\right)] = 0;$$
	\item {\bf Context desideratum.} Match posterior distributions: $$\Ebb_{\tau \sim p_\omega(\tau)}[D_{\KL}\left(p_\omega(m \vert \tau)|| q_\psi(m \vert \tau)\right)] = 0.$$
\end{enumerate}
The reward desideratum aligns with the objective in MF-AIRL, encouraging the trajectory distribution induced by the learned reward function to match the distribution of demonstrations. The context desideratum ensures a match between the context variable inference model and the posterior distribution induced by the learned reward function, facilitating $q_\psi(m \vert \tau)$ to serve as a suitable variational approximation to the unknown $p_\omega(m \vert \tau)$. This allows correct inference of the context variable from a new demonstrated trajectory. Combining mutual information as the optimisation objective and the two aforementioned desiderata as constraints, we arrive at the final formula for our target optimization problem, with its Lagrangian dual function being:
\begin{equation*}
\begin{aligned}
	\min_{\omega, \psi}   &~ \alpha \cdot \Ebb_{m\sim p(m)}[D_{\KL}\left(p_{\muv_E, \piv_E}(\tau \vert m)|| p_\omega(\tau \vert m)\right)] + \\
	  &~ \beta \cdot \Ebb_{\tau \sim p_\omega(\tau)}[D_{\KL}\left(p_\omega(m \vert \tau)|| q_\psi(m \vert \tau)\right)] -I_{p_\omega(m;\tau)}.
\end{aligned}
\end{equation*}

By fixing Lagrangian multipliers to specific values ($\alpha = \beta = 1$) that confirm the incentive to maximise the mutual information \citep{zhao2018information,yu2019meta}, we can rewrite the above Lagrangian dual function as
\begin{equation}\label{eq:Linfo}
\begin{aligned}
	 \min_{\omega, \psi} ~& \Ebb_{m \sim p(m)}[D_{\KL}\left(p_{\muv_E, \piv_E}(\tau \vert m)|| p_\omega(\tau \vert m)\right)] - \\
	 & \Ebb_{m \sim p(m), \tau \sim p_\omega(\tau \vert m)}[\log q_\psi(m \vert \tau)].
\end{aligned}
\end{equation}
The detailed derivation is given in Appendix~\ref{app:target-original}.

\section{Probabilistic Embeddings for Meta-MFIRL}
To optimise the objective in Eq.~\eqref{eq:Linfo} tractably, we need to replace the unknown $\mu_\omega^t(\cdot \vert m)$ (part of $p_\omega(\tau \vert m)$) with a known value that retains optimality in ideal conditions. One approach is to imitate the operation in MF-AIRL, using an estimate $\hat{\mu}_E^t(\cdot \vert m)$ and substituting it for $\mu_\omega^t(\cdot \vert m)$ (Eq.\eqref{eq:empirical-mf-mfirl}). However, the dependency of $\mu_\omega^t(\cdot \vert m)$ on $m$ prevents this substitution, as we lack access to the posterior distribution $p(m \vert \tau)$.
Fortunately, we can use $q_\psi(m \vert \tau)$ as a variational approximation to $p(m \vert \tau)$ to account for the uncertainty over tasks. Using this approximation, we construct an estimate of $\mu_E^t(\cdot \vert m)$ according to the following rule:
\begin{equation}\label{eq:empirical-mf}
	\hat{\mu}_\psi^t(s \vert m) = \Ebb_{\tau_E \sim p_{\muv_E, \piv_E}(\tau)} \left[ q_\psi(m \vert \tau_E) \cdot \mathds{1}_{s^t = s} \right],
\end{equation}
which is unbiased when $\psi$ is trained to the optimality. 
Note that the replacement of $\mu_\omega^t(\cdot \vert m)$ with $\hat{\mu}_\psi^t(s \vert m)$ will make $p_\omega(\tau \vert m)$ additionally depend on $\psi$. We rewrite the resulting conditional distribution as 
\begin{equation*}\label{eq:p_omega_psi}
	p_{\omega,\psi}(\tau \vert m) = \frac{1}{Z(\omega,\psi)}\left[ \prod_{t=0}^T \hat{\mu}_\psi^t(s^t) \right] \cdot e^{\sum_{t=0}^T f_\omega\left(s^t, a^t, \hat{\mu}_\psi^t, m \right)}.
\end{equation*}
Accordingly, our original target problem in Eq.~\eqref{eq:Linfo} now takes the following form:
\begin{equation}\label{eq:Lkl-Linfo}
	\min_{\omega,\psi} \Kmc(\omega,\psi) - \Lmc(\omega,\psi)
\end{equation}
\begin{equation*}
\begin{aligned}
	\Kmc(\omega,\psi) &= \Ebb_{m \sim p(m)}[D_{\KL}\left(p_{\muv_E, \piv_E}(\tau \vert m)|| p_{\omega,\psi}(\tau \vert m)\right)],\\
	\Lmc(\omega, \psi) &= \Ebb_{m \sim p(m), \tau \sim p_{\omega,\psi}(\tau \vert m)}[\log q_\psi(m \vert \tau)].
\end{aligned}
\end{equation*}

We next introduce how to approximately optimise the objective in Eq.~\eqref{eq:Lkl-Linfo} with sampling-based gradient estimation. First, we observe that $\Kmc(\omega,\psi)$ with fixed $\psi$ can be maximised using the adversarial framework in MF-AIRL, where the adaptive sampler (policy) takes the forms of $\pi_\theta(a \vert s,m)$. This observation can be formally stated in the following lemma that will be useful to derive the gradients of $\Lmc(\omega,\psi)$.

\begin{algorithm}[!htp]
   \caption{PEMMFIRL Meta-Training}\label{alg:PEMMFIRL}
\begin{algorithmic}[1]
   \STATE {\bf Input:} Expert trajectories $\Dmc_E = \{ \tau_j \}_{j = 1}^M$.
   \STATE {\bf Initialisation:} Parameters $f_\omega$, $q_\psi$, $\piv_\theta$.
   \FOR{each iteration}
   		\STATE Sample two set of trajectories $\tau_E, \tau_E' \sim 
   		\Dmc_E$.
   		\STATE Infer a batch of context variables $\tilde{m} \sim q_\psi(m \vert \tau_E)$.
   		\STATE Estimate a batch of mean fields $\hat{\mu}_\psi(s \vert \tilde{m})$ with Eq.~\eqref{eq:empirical-mf}.
   		\STATE Generate a set of trajectories $\Dmc$ using $\piv_\theta(a \vert s,\tilde{m})$ and $\hat{\mu}_\psi(s \vert \tilde{m})$ with the fixed $\tilde{m}$ for each trajectory.
   		\STATE Update $\psi$ to decrease $\Kmc(\omega,\psi) - \Lmc(\omega,\psi)$ with gradients $\frac{\partial \Kmc}{\partial \psi}-\frac{\partial \Lmc}{\partial \psi}$ estimated on $\Dmc$.
   		\STATE Update $\omega$ to increase  $\Lmc(\omega,\psi)$ with $\frac{\partial \Lmc}{\partial \omega}$ estimated on $\Dmc$.
   		\STATE Update $\omega$ to increase following objective:\\ $\Ebb_{\tilde{m},\tau_E'}[\sum_{t=0}^T \log D_\omega ] + \Ebb_{\tilde{m}, \tau \sim \Dmc}[\sum_{t=0}^T \log (1 - D_\omega)]$.
   		\STATE Update $\theta$ with RL to increase the following objective: $\Ebb_{\tilde{m},\tau_E'}[\sum_{t=0}^T f_\omega - \log \pi_\theta^t(a^t \vert s^t,\tilde{m})]$.
   \ENDFOR
   \STATE {\bfseries Output:} Reward function $f_{\omega}$ and inference model $q_\psi$.
\end{algorithmic}
\end{algorithm}

\begin{lemma}\label{lem:MFIRL}
	Let the adversarial framework of MF-AIRL take the adaptive samplers $\pi_\theta^t(a \vert s,m)$ and the discriminator $D_\omega(s,a,\hat{\mu}_\psi^t, m)=\frac{\exp\left( f_\omega(s,a, \hat{\mu}_\psi^t)\right)}{\exp\left( f_\omega(s,a,\hat{\mu}_\psi^t) + \pi_\theta(a \vert s,m)\right)}$. If the adaptive samplers are trained to the optimility $\piv_\theta^*$ w.r.t. the reward signal $\log D_\omega - \log (1-D_\omega)$, then the trajectory distribution induced by $(\hat{\muv}_\psi, \piv_\theta^*)$ matches the conditional distribution $p_{\omega,\psi}(\tau \vert m)$, i.e., $p_{\hat{\muv}_\psi, \piv_\theta^*}(\tau \vert m) = p_{\omega,\psi}(\tau \vert m)$.
\end{lemma}
Lemma~\ref{lem:MFIRL} tells us we can instead use $\piv_\theta^*$  to generate trajectory samples from the energy-based model $p_{\omega,\psi}(\tau \vert m)$ that is difficult to directly sample.  Now we are ready to estimate the gradients of $\Lmc(\omega,\psi)$ w.r.t. $\omega$ and $\psi$.

\begin{proposition}\label{prop:Linfo}
	With $m\sim p(m)$, $\hat{\tau}$ and $\hat{\tau}'\sim p_{\hat{\muv}_\psi, \piv_\theta^*}(\tau \vert m)$ and $f_\omega$ denoting $f_\omega(s^t,a^t,\hat{\mu}^t_\psi,m)$ for short, the gradients of $\Lmc(\omega, \psi)$ w.r.t. $\omega$ and $\psi$ can be estimated with:
	\begin{equation*}
	\small
	\begin{aligned}
		&\Ebb_{m,\hat{\tau}} \left[\log q_\psi(m \vert \hat{\tau}) \left[  \sum_{t=0}^T \frac{\partial f_\omega}{\partial \omega}  -\Ebb_{\hat{\tau}'}\left[\sum_{t=0}^T \frac{\partial f_\omega}{\partial \omega}   \right] \right]\right] \text{ and } \Ebb_{m,\hat{\tau}} \Big[\\
		& \log q_\psi(m \vert \hat{\tau}) (\kappa(\hat{\tau},m)-\Ebb_{\hat{\tau}'}[\kappa(\hat{\tau}',m)] ) +  \frac{\partial \log q_\psi(m \vert \hat{\tau}) }{\partial \psi}  \Big],
	\end{aligned}
	\end{equation*}
	\begin{equation*}
	\small
		\text{ where } \kappa(\tau, m) = \sum_{t=0}^T \left[ \left(\frac{\partial f_\omega}{\partial \hat{\mu}_\psi^t} + \frac{1}{\hat{\mu}_\psi^t(s^t \vert m)}\right) \frac{\partial \hat{\mu}_\psi^t(s^t \vert m)}{\partial \psi} \right]
	\end{equation*}
\end{proposition}
\begin{proof}
	See Appendix~\ref{app:Linfo}.
\end{proof}

Note that $\frac{\partial f_\omega}{\partial \hat{\mu}_\psi^t}$ is a function of $\omega$ and depends on the model ({\em e.g.,} neural networks) for $f_\omega$, and calculating $\frac{\partial \hat{\mu}_\psi^t(s^t \vert m)}{\partial \psi}$ is straightforward according to its definition in Eq.~\eqref{eq:empirical-mf}.
Also note that the expectations in both estimates above are taken over the prior  $p(m)$. Since we have no access to $p(m)$ but have a set of demonstrated trajectories, we can generate synthetic samples for $m$ using the following generative process
\begin{equation}\label{eq:generative}
	\tau_E \sim p_{\muv_E, \piv_E}(\tau), \tilde{m} \sim q_\psi(m \vert \tau_E),
\end{equation}  
which matches $p(m)$ when $\omega$ and $\psi$ are trained to optimality. 
We are left to estimate the gradient of $\Kmc(\omega, \psi)$ w.r.t. $\psi$.

\begin{proposition}\label{prop:Kinfo}
	With $\tilde{\tau} \sim p_{\hat{\muv}_\psi, \piv_\theta^*}(\tau \vert \tilde{m})$ and $\tilde{m}$ by Eq.~\eqref{eq:generative}, the gradient of $\Kmc(\omega, \psi)$ w.r.t. $\psi$ can be estimated with:
	\begin{equation*}
		\Ebb_{\tau_E, \tilde{m}}[\Ebb_{\tilde{\tau}}[ \kappa(\tilde{\tau}, \tilde{m}) ] - \kappa(\tau_E,\tilde{m}) ],
	\end{equation*}
\end{proposition}
\begin{proof}
	See Appendix~\ref{app:Kinfo}.
\end{proof}

Now we are ready to introduce our proposed framework termed Probabilistic Embeddings for Meta-Mean Field IRL (PEMMFIRL) where we alternatively update an $\omega$-parameterised reward function $f_\omega(s,a,\mu,m)$ and a $\psi$-parameterised context variable inference model $q_\psi(m \vert \tau)$. We update $\omega$ by $\frac{\partial}{\partial \omega} \Lmc(\omega, \psi)$ and invoking MF-AIRL to solve $\Kmc(\omega, \psi)$ given the current $\psi$, where we train a discriminator $D_\omega$ given in Lemma~\ref{lem:MFIRL} and adaptive samplers $\piv_\theta$ as 
\begin{equation*}\label{eq:discriminator}
\begin{aligned}
	\max_\omega \Ebb_{\tau_E,\tilde{m}}\left[\sum_{t=0}^T \log D_\omega \right] + \Ebb_{\tilde{m}, \hat{\tau}}\left[\sum_{t=0}^T \log (1 - D_\omega)\right]
\end{aligned}
\end{equation*} 
\begin{equation*}\label{eq:sampler}
\begin{aligned}
	\text{and } \max_\theta \Ebb_{\tilde{m}, \hat{\tau}}\left[\sum_{t=0}^T f_\omega - \log \pi_\theta^t(a^t \vert s^t,\tilde{m}) \right].
\end{aligned}
\end{equation*}
The update of $\psi$ is according to $\frac{\partial}{\partial \psi} \Kmc(\omega, \psi)$ and $\frac{\partial}{\partial \psi} \Lmc(\omega, \psi)$. As a summary, we present the pseudocode of PEMMIRL meta-training in Alg.~\ref{alg:PEMMFIRL}. The pseudocode of the meta-test is deferred until Appendix~\ref{app:meta-test}.

%% file: experiments.tex
\section{Experiments}
We experimentally seek answers to two critical questions: (1) {\em Can PEMMFIRL accurately recover task type and reward function from demonstrated trajectories of multiple tasks?} (2) {\em Can PEMMFIRL effectively learn good policies for each task type?} We evaluate our algorithm on simulated MFG environments and a real-world spatial pricing problem for taxi rides. Detailed task descriptions are provided in Appendix~\ref{app:task}, while specific settings for dataset preprocessing, hyper-parameters, network architectures, and hardware environments can be found in Appendix~\ref{app:exp}.

\subsection{Simulated Tasks}
\noindent{\bf Task Descriptions.} There are three simulated MFG environments ordered in increasing complexity: 
	(1) {\bf Virus infection} (VIRUS) \citep{cui2021approximately}: Agents choose between ``social distancing'' or ``going out'', affecting susceptible, infection and recovery probabilities.
	(2) {\bf Malware spread} (MALWARE) \citep{huang2017mean}: This model is representative of several problems with positive externalities, such as flu vaccination and economic models involving the entry and exit of firms.
	(3) {\bf Investment in product quality} (INVEST) \citep{subramanian2019reinforcement}: This model captures investment decisions in a fragmented market with multiple competing firms producing the same product.
\smallskip

\noindent{\bf Settings.} For each task, we define different reward functions $r(s,a,\mu,m)$ based on context variables $m$. As in \citep{chen2022individual}, for each $m$, the expert training process involves iteratively updating $\muv_E(\cdot \vert m)$ and $\piv_E(\cdot \vert m)$ until convergence: $\piv_E$ is updated using backward induction, while $\muv_E$ is updated using the MKV equation (Eq.~\eqref{eq:MKV}). Demonstrated trajectories (each with a length of 50) are generated using $(\muv_E(\cdot \vert m),\piv_E(\cdot \vert m))$ with $m$ drawn from a uniform distribution $p(m)$. The IRL algorithms have no access to $r(s,a,\mu,m)$, $m$, or $p(m)$.

\smallskip

\noindent{\bf Baselines.} Two state-of-the-art methods are compared:
	 {\bf (1) Population-level IRL (PLIRL)} \citep{yang2018learning}:  A centralised method that converts MFG to MDP and performs IRL on this MDP. It assumes cooperation among agents to maximise the population's average rewards. 
	 {\bf (2) Mean Field Adversarial IRL (MF-AIRL)} \citep{chen2023adversarial}: This method leverages adversarial learning to efficiently infer rewards for entropy-regularised MFGs.

\smallskip

\noindent{\bf Performance Metrics.}
The quality of a learned reward function $f_\omega$ is evaluated by comparing its induced $\piv_\omega$ with the expert policy $\piv_E$. Two metrics are used to measure this difference -- statistical distance and cumulative rewards: 

\begin{enumerate}
	\item {\em Policy deviation}. The expectation over KL-divergence: $\Ebb_{m \sim p(m)} [ \sum_{t = 0}^T \sum_{s \in \Smc} D_{\KL} ( \pi_E^t (\cdot \vert s,m) \parallel \pi_\omega^t (\cdot \vert s,m) )]$ measures the statistical distance between two policies.
    \item {\em Expected return}. The difference between two expected returns of $(\muv_\omega, \piv_\omega)$ and $(\muv_E, \piv_E)$ under the ground-truth reward function $r(s,a,\mu,m)$ and the prior $p(m)$. 
\end{enumerate}

\begin{figure}
	\centering
\includegraphics[width=.47\textwidth]{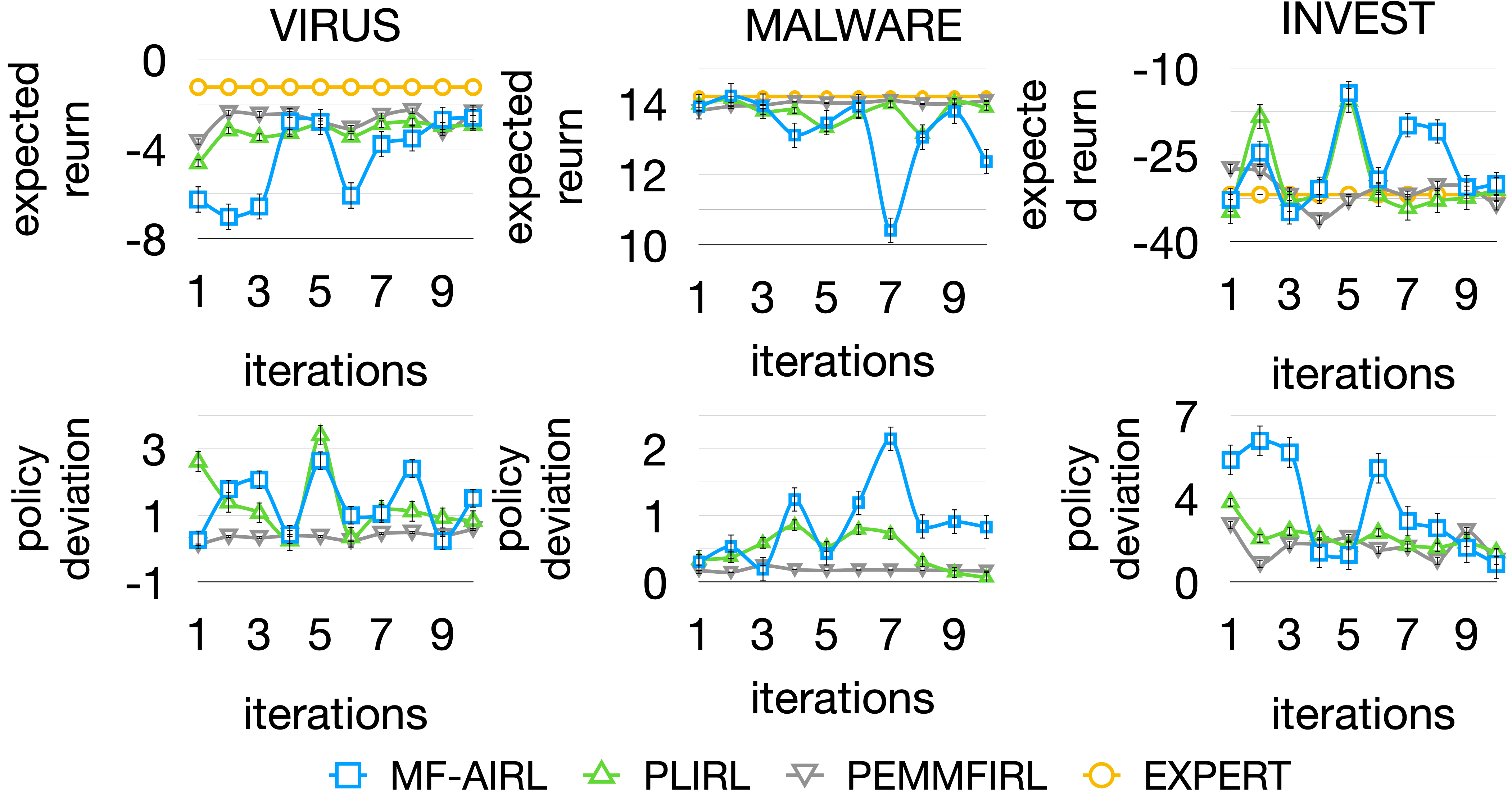}
	\caption{Results for simulated tasks. The curves and bars are the median and variance over ten independent runs.
	}\label{fig:simulated}
\end{figure}

\noindent{\bf Results.} 
As depicted in Fig.~\ref{fig:simulated}, the learned policy by PEMMFIRL shows a minor deviation from the expert policy, leading to a slight difference in expected return under the ground-truth reward function. Since PLIRL and MF-AIRL cannot handle different context variables, they both exhibit large deviations. Moreover, PLIRL shows the largest deviation as it is not suitable for non-fully cooperative environments. PEMMFIRL experiences a slight decrease in time performance due to the additional input parameters required for the context inference model $q(m|\tau)$. However, this performance decrease remains acceptable. Overall, PEMMFIRL performs well in various numerical experimental tasks, producing accurate results similar to expert performance. It demonstrates good stability during iterations, with a significantly lower average variance (32\%) in expected returns and policy deviations than other algorithms, indicating low fluctuations during the iterations.

\subsection{Spatial Pricing for Taxi Rides}

\noindent{\bf Task Description.} 
Spatial pricing for taxi rides aims to implement different pricing strategies in diverse areas to address the mismatched supply and demand in taxi services. By applying higher pricing in areas with high demand and low supply, taxis are encouraged to move to those regions when empty, increasing availability. The problem can be addressed using mean-field games, where drivers make decisions based on passenger density \cite{ata2019spatial}. 
By taking the value of PEMMIRL in the sense of pricing by differentiating drivers' personal preferences, we seek to improve drivers' profits further.

\smallskip

\begin{figure}
	\centering
	\includegraphics[width=.47\textwidth]{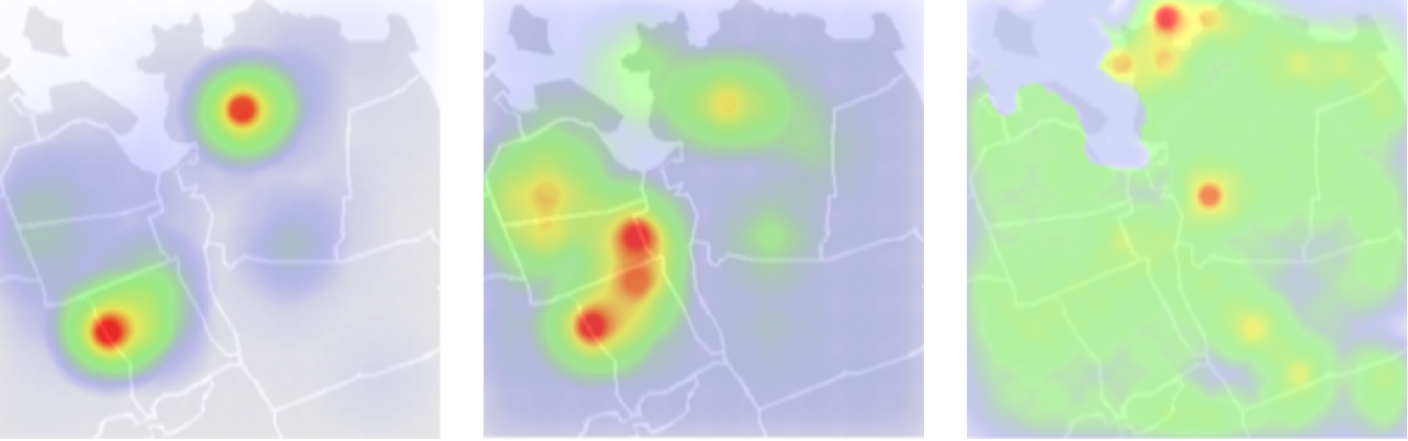}
	\caption{Illustration of the initial setup. From left to right: distribution of passengers, distribution of taxis, and distribution of passengers' travelling distances.}\label{fig:setup}
\end{figure}

\begin{figure}
	\centering
	\includegraphics[width=.47\textwidth]{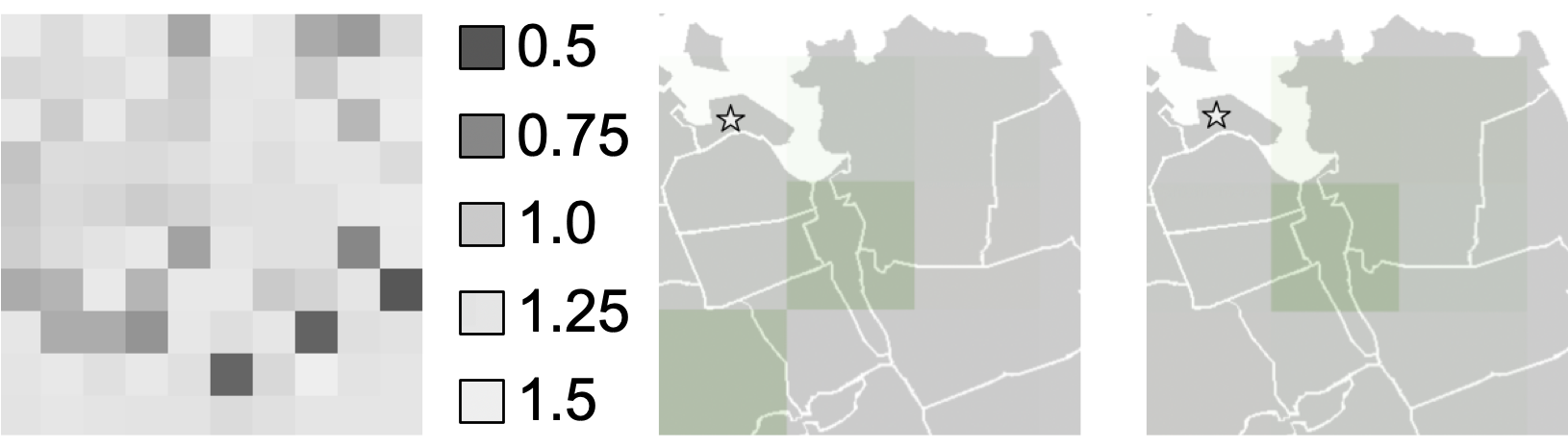}
	\caption{Illustration for the underlying price multipliers and the learned policies (with different context variables) at the initial step for drivers at the starred location. A driver takes passengers in a darker area with a higher probability.}\label{fig:policy} 
\end{figure}

\noindent{\bf Settings.} Real-world taxi operation records, including passenger travel time, departure locations, and destinations, are used as experimental data for passenger travel demand. This data is extracted from the New York Yellow Taxi Dataset \cite{ata2019spatial}. To ensure sufficient data in each segmented node, we select a subset of the data with one-month records that fall within part of Queens borough of New York City (see Fig.~\ref{fig:setup}) and divide it into a 10 $\times$ 10 grid using a boundary of 0.01$^\circ$ for longitude and latitude.

\citet{ata2019spatial} developed an empirical MFG model that fits the dataset. The driver's reward function calculates revenue considering destinations and deducts movement-related losses, such as fuel consumption and vehicle depreciation. It applies a surcharge based on the driver's origin. The origin-only reward function reported in \cite{ata2019spatial} is expressed as follows:
$$r(s=i,a=j,\mu) = (\eta^{0.5265} - \eta_s^{0.5265}) \cdot f(i,j,\mu),$$
where $i$ and $j$ represent the driver's origin area and the passenger's location, respectively; $\mu$ denotes the current density of empty taxis; $f$ is a fixed function to calculate the standard profit based on the passenger demand and passengers' travelling distances; $\eta$ is a constant upper limit for the price passengers are not willing to pay more than;  $\eta_s$ is referred to as the {\em price multiplier}, varying across regions to determine the surcharge scale of each area. As in \cite{ata2019spatial}, we set $\eta=2.33$ and adopt $\eta_s$ shown in Fig.~\ref{fig:policy}.

To additionally consider the dependency of the reward function on the context variable, we assume a context-conditioned profit  $f(i,j,\mu,m)$ (modelled by a neural network) with a discrete context variable $m \in \{1,2\}$. Intuitively, we believe $m$ captures the impact of drivers' preferences, {\em e.g.,} trip distances, over profits. 
A trajectory has a total of 120 steps, and each step simulates five minutes. 
The initial setup for trajectory sampling and reward (profit) calculation in terms of distributions of the passengers, the taxis and passengers' travelling distances is illustrated in Fig.~\ref{fig:setup}.

\smallskip


\noindent{\bf Baseline.} We compare all drivers' average profit induced by the learned policy of PEMMFIRL with that induced by the optimal policy of the empirical MFG model \citep{ata2019spatial}, under the context-conditioned profit.

\smallskip


\noindent{\bf Results.} 
Fig.~\ref{fig:policy} demonstrates the learned policy by PEMMFIRL based on the inferred rewards functions. Applying this new pricing strategy and the corresponding policy to actual passenger demand data decreases served passengers compared to the original data. This is expected because no additional passenger demands are introduced beyond the dataset. 
By adjusting the upper limit of price $\eta$, different decay rates of passengers and increased rates (compared to the profits under the empirical MFG model in \citep{ata2019spatial}) of drivers' average profit are obtained. 
Note that changing $\eta$ does not affect the optimal policy, as it is a constant. The results in Tab.~\ref{tab:rates} indicate that for a relatively small decay proportion, the model increases the average profit of all drivers by 2.8\% (\$0.1308 per ride) with a 0.4\% reduction in passenger numbers. Further analysis reveals that most reduced passengers are those on short trips. For a higher decay proportion, the model increases the average profit of all drivers by 3.1\% (\$0.1448 per ride) with a 0.7\% reduction in passenger numbers. Overall, this suggests that the learned policy effectively raises the profit of taxi drivers. 

\begin{table}
\centering
\small
	\begin{tabular}{c c c c c}
	\toprule
		Value & Decay rate of & Increase rate of & Increased fare\\
		 of $\eta$ & served passengers & average profit &  per ride\\
		\midrule
		5 & - 0.4\% & +2.8\% & +\$0.1308\\
		10 & - 0.5\% & +2.3\% & +\$0.1074\\
		15 & - 0.6\% & +3.4\% & +\$0.1589\\
		20 & - 0.7\% & +3.1\% & +\$0.1448\\
	\bottomrule
	\end{tabular}
	\caption{Increase rate of the average profit and increased fare per ride under different decay rates of served passengers.}\label{tab:rates}
\end{table}

%% file: conclusions.tex
\section{Conclusions}

This paper addresses the incapacity of IRL  in the face of a large number of agents with distinct and unknown rewards. By introducing probabilistic contextual variables, we extend MFGs to handle heterogeneous agents without altering mean-field approximation. Based on this generalisation, we develop the Probabilistic Embeddings for Meta-Mean Field IRL (PEMMFIRL), allowing inferring reward functions from mixed-type demonstrations without prior knowledge of types. Experiments on simulated tasks and a real-world taxi-ride pricing problem reveal improvements over state-of-the-art IRL methods in MFGs.

\section{Acknowledgments}
This work was supported by a grant from the New Zealand Tertiary Education Commission and by the Strong AI Lab at the University of Auckland.

%% file: appendix.tex
\newpage
\onecolumn
\setcounter{secnumdepth}{1}
\setcounter{section}{0}
\renewcommand\thesection{\Alph{section}}
\setcounter{equation}{0}
\renewcommand{\theequation}{\thesection\arabic{equation}} 

\appendix
\section{The Deviation of the Lagrangian Dual Function with Fixed Multipliers}\label{app:target-original}
~
 
\begin{equation*}
\begin{aligned}
	& \min_{\omega, \psi}    \Ebb_{m\sim p(m)}[D_{\KL}\left(p_{\muv_E, \piv_E}(\tau \vert m)|| p_\omega(\tau \vert m)\right)] +  \Ebb_{\tau \sim p_\omega(\tau)}[D_{\KL}\left(p_\omega(m \vert \tau)|| q_\psi(m \vert \tau)\right)] -I_{p_\omega(m;\tau)}\\
	= & \min_{\omega, \psi}    \Ebb_{m\sim p(m)}[D_{\KL}\left(p_{\muv_E, \piv_E}(\tau \vert m)|| p_\omega(\tau \vert m)\right)] + \Ebb_{m \sim p(m), \tau \sim p_\omega(\tau \vert m)}\left[ \log p_\omega(m \vert \tau) - q_\psi(m \vert \tau) \right]\\
	& ~~~~~~~~~~~~~ - \Ebb_{m\sim p(m), \tau \sim p_\omega(\tau \vert m)} \left[ \log p_\omega(m \vert \tau) - \log p(m) \right]\\
	 = & \min_{\omega, \psi}~  \Ebb_{m \sim p(m)}[D_{\KL}\left(p_{\muv_E, \piv_E}(\tau \vert m)|| p_\omega(\tau \vert m)\right)] +  \Ebb_{m \sim p(m), \tau \sim p_\omega(\tau \vert m)}\left[ \log \frac{p(m)}{p_\omega(m \vert \tau)} + \log \frac{p_\omega(m \vert \tau)}{q_\psi(m \vert \tau)} \right]\\
	 = & \min_{\omega, \psi}  \Ebb_{m \sim p(m)}[D_{\KL}\left(p_{\muv_E, \piv_E}(\tau \vert m)|| p_\omega(\tau \vert m)\right)] - \Ebb_{m \sim p(m), \tau \sim p_\omega(\tau \vert m)}[\log q_\psi(m \vert \tau) - \log p(m)],\\
	 = & \min_{\omega, \psi} \Ebb_{m \sim p(m)}[D_{\KL}\left(p_{\muv_E, \piv_E}(\tau \vert m)|| p_\omega(\tau \vert m)\right)] -  \Ebb_{m \sim p(m), \tau \sim p_\omega(\tau \vert m)}[\log q_\psi(m \vert \tau)].  
\end{aligned}
\end{equation*}
The term $-\log p(m)$ in the third equation is omitted as it is a constant.

\section{Proof of Proposition~\ref{prop:Linfo}}\label{app:Linfo}

The gradient of $\Lmc(\omega, \psi)$ w.r.t. $\omega$ is expressed as

\begin{equation*}
	\begin{aligned}
	\frac{\partial}{\partial \omega} \Lmc(\omega, \psi) = & \frac{\partial}{\partial \omega} \Ebb_{m \sim p(m), \tau \sim p_{\omega,\psi}(\tau \vert m)}[\log q_\psi(m \vert \tau)]\\
	= & \Ebb_{m \sim p(m), \tau \sim p_{\omega,\psi}(\tau \vert m)} \left[ \log q_\psi(m \vert \tau) \frac{\partial}{\partial \omega} \log p_{\omega, \psi}(\tau \vert m) \right]\\
	= & \Ebb_{m \sim p(m), \tau \sim p_{\omega,\psi}(\tau \vert m)} \left[ \log q_\psi(m \vert \tau) \left(  \sum_{t=0}^T \frac{\partial}{\partial \omega} f_\omega - \frac{\partial}{\partial \omega} \log Z(\omega, \psi) \right) \right]\\
	= & \Ebb_{m \sim p(m), \tau \sim p_{\omega,\psi}(\tau \vert m)} \left[  \sum_{t=0}^T \frac{\partial}{\partial \omega} f_\omega - \Ebb_{\tau' 
 \sim p_{\omega,\psi}(\tau \vert m)} \left[ \sum_{t=0}^T \frac{\partial}{\partial \omega} f_\omega \right] \right]
	\end{aligned}
\end{equation*}

Generating trajectory samples from the energy-based $p_{\omega,\psi}(\tau \vert m)$ using $\piv_\theta^*$ (see Lemma~\ref{lem:MFIRL}), we get the estimate as given in Propostion~\ref{prop:Linfo}.

The gradient of $\Lmc(\omega, \psi)$ w.r.t. $\psi$ can be written as

\begin{equation*}
	\begin{aligned}
	\frac{\partial}{\partial \psi} \Lmc(\omega, \psi) = & \frac{\partial}{\partial \psi} \Ebb_{m \sim p(m), \tau \sim p_{\omega,\psi}(\tau \vert m)}[\log q_\psi(m \vert \tau)]\\
	= & \Ebb_{m \sim p(m), \tau \sim p_{\omega,\psi}(\tau \vert m)} \left[ \log q_\psi(m \vert \tau) \frac{\partial}{\partial \psi} \log p_{\omega, \psi}(\tau \vert m) +  \frac{\partial}{\partial \psi} \log q_\psi(m \vert \tau) \right]\\
	= & \Ebb_{m \sim p(m), \tau \sim p_{\omega,\psi}(\tau \vert m)} \Bigg[ \log q_\psi(m \vert \tau) \left( \sum_{t=0}^T \left[ \left(\frac{\partial f_\omega}{\partial \hat{\mu}_\psi^t} + \frac{1}{\hat{\mu}_\psi^t(s^t \vert m)}\right) \frac{\partial \hat{\mu}_\psi^t(s^t \vert m)}{\partial \psi} \right] - \frac{\partial}{\partial \psi} \log Z(\omega, \psi)  \right)\\
	& \qquad\qquad\qquad\qquad\qquad +  \frac{\partial}{\partial \psi} \log q_\psi(m \vert \tau) \Bigg]\\
	= & \Ebb_{m \sim p(m), \tau \sim p_{\omega,\psi}(\tau \vert m)} \Bigg[ \log q_\psi(m \vert \tau) \Bigg( \sum_{t=0}^T \left[ \left(\frac{\partial f_\omega}{\partial \hat{\mu}_\psi^t} + \frac{1}{\hat{\mu}_\psi^t(s^t \vert m)}\right) \frac{\partial \hat{\mu}_\psi^t(s^t \vert m)}{\partial \psi} \right] \\
	& \qquad\qquad\qquad\qquad\qquad -\Ebb_{\tau' 
 \sim p_{\omega,\psi}(\tau \vert m)} \left[ \sum_{t=0}^t \left(\frac{\partial f_\omega}{\partial \hat{\mu}_\psi^t} + \frac{1}{\hat{\mu}_\psi^t(s^t \vert m)}\right) \frac{\partial \hat{\mu}_\psi^t(s^t \vert m)}{\partial \psi} \right]\Bigg)\\
	& \qquad\qquad\qquad\qquad\qquad +  \frac{\partial}{\partial \psi} \log q_\psi(m \vert \tau) \Bigg]\\ 
	\end{aligned}
\end{equation*}

By taking the nottaion $\kappa(\tau,m) = \sum_{t=0}^T \left[ \left(\frac{\partial f_\omega}{\partial \hat{\mu}_\psi^t} + \frac{1}{\hat{\mu}_\psi^t(s^t \vert m)}\right) \frac{\partial \hat{\mu}_\psi^t(s^t \vert m)}{\partial \psi} \right]$, we can rewrite the equation above as

\begin{equation*}
	\frac{\partial}{\partial \psi} \Lmc(\omega, \psi) = \Ebb_{m\sim p(m),\tau \sim p_{\omega,\psi}(\tau \vert m)} \left[\log q_\psi(m \vert \tau) \left(\kappa(\tau)-\Ebb_{\tau' \sim p_{\omega,\psi}(\tau \vert m) }[\kappa(\tau')] \right) + \frac{\partial}{\partial \psi} \log q_\psi(m \vert \tau) \right]
\end{equation*}

Again, generating trajectory samples from the energy-based $p_{\omega,\psi}(\tau \vert m)$ using $\piv_\theta^*$, we get the estimate as given in Propostion~\ref{prop:Linfo}.

\begin{equation*}
	\frac{\partial}{\partial \psi} \log p_{\omega, \psi}(\tau \vert m) = \sum_{t=0}^T \left[ \left(\frac{\partial f_\omega}{\partial \hat{\mu}_\psi^t} + \frac{1}{\hat{\mu}_\psi^t(s^t \vert m)}\right) \frac{\partial \hat{\mu}_\psi^t(s^t \vert m)}{\partial \psi} \right] - \frac{\partial}{\partial \psi} \log Z(\omega, \psi)
\end{equation*}

\section{Proof of Proposition~\ref{prop:Kinfo}}\label{app:Kinfo}
The gradient of $\Kmc(\omega, \psi)$ w.r.t. $\psi$ is expressed as
\begin{equation*}
\begin{aligned}
	\frac{\partial}{\partial \psi}\Kmc(\omega, \psi) = & \frac{\partial}{\partial \psi} \Ebb_{m \sim p(m)}[D_{\KL}\left(p_{\muv_E, \piv_E}(\tau \vert m)|| p_{\omega,\psi}(\tau \vert m)\right)]\\
	= & \frac{\partial}{\partial \psi} \Ebb_{m \sim p(m), \tau \sim p_{\muv_E, \piv_E}(\tau \vert m)} \left[ \log p_{\muv_E, \piv_E}(\tau \vert m) - \log p_{\omega,\psi}(\tau \vert m)   \right]\\
	= &  - \Ebb_{m \sim p(m), \tau \sim p_{\muv_E, \piv_E}(\tau \vert m)} \left[ \frac{\partial}{\partial \psi} \log p_{\omega,\psi}(\tau \vert m)   \right]\\
	= & \Ebb_{m \sim p(m), \tau \sim p_{\muv_E, \piv_E}(\tau \vert m)} \left[ \frac{\partial}{\partial \psi} \log Z(\omega, \psi) - \sum_{t=0}^T \left[ \left(\frac{\partial f_\omega}{\partial \hat{\mu}_\psi^t} + \frac{1}{\hat{\mu}_\psi^t(s^t \vert m)}\right) \frac{\partial \hat{\mu}_\psi^t(s^t \vert m)}{\partial \psi} \right]  \right]\\
	= & \Ebb_{m \sim p(m), \tau \sim p_{\muv_E, \piv_E}(\tau \vert m)} \Bigg[ \Ebb_{\tau' 
 \sim p_{\omega,\psi}(\tau \vert m)} \left[ \sum_{t=0}^t \left(\frac{\partial f_\omega}{\partial \hat{\mu}_\psi^t} + \frac{1}{\hat{\mu}_\psi^t(s^t \vert m)}\right) \frac{\partial \hat{\mu}_\psi^t(s^t \vert m)}{\partial \psi} \right] - \\
 	& \qquad\qquad\qquad\qquad\qquad \sum_{t=0}^T \left[ \left(\frac{\partial f_\omega}{\partial \hat{\mu}_\psi^t} + \frac{1}{\hat{\mu}_\psi^t(s^t \vert m)}\right) \frac{\partial \hat{\mu}_\psi^t(s^t \vert m)}{\partial \psi} \right] \Bigg].
\end{aligned} 
\end{equation*}

By taking the nottaion $\kappa(\tau,m) = \sum_{t=0}^T \left[ \left(\frac{\partial f_\omega}{\partial \hat{\mu}_\psi^t} + \frac{1}{\hat{\mu}_\psi^t(s^t \vert m)}\right) \frac{\partial \hat{\mu}_\psi^t(s^t \vert m)}{\partial \psi} \right]$, we can rewrite the equation above as

\begin{equation*}
	\frac{\partial}{\partial \psi}\Kmc(\omega, \psi) = \Ebb_{m \sim p(m), \tau \sim p_{\muv_E, \piv_E}(\tau \vert m)}[\Ebb_{\tau' 
 \sim p_{\omega,\psi}(\tau \vert m)}[ \kappa(\tau') ] - \kappa(\tau) ].
\end{equation*}

By generating synthetic samples for $m$ using the generative process in Eq.~\eqref{eq:generative} and generating trajectory samples from the energy-based $p_{\omega,\psi}(\tau \vert m)$ using $\piv_\theta^*$, we cam further write the equation above as 
\begin{equation*}
	\frac{\partial}{\partial \psi}\Kmc(\omega, \psi) = \Ebb_{\tau_E \sim p_{\muv_E, \piv_E}(\tau), \tilde{m} \sim q_\psi(m \vert \tau_E)} \left[ \Ebb_{\tilde{\tau} \sim p_{\hat{\muv}_\psi, \piv_\theta^*}(\tau \vert \tilde{m})}[\kappa(\tilde{\tau}, \tilde{m})] - \kappa(\tau_E, \tilde{m})  \right].
\end{equation*}

\section{Pesudo-code of Meta-test of PEMMFIRL}\label{app:meta-test}

\begin{algorithm}
\caption{PEMMFIRL Meta-test}
\begin{algorithmic}[1]
	\STATE {\bf Input:} A context variable $m \sim p(m)$, an expert demonstration $\tau_E \sim p_{\muv_E,\piv_E}(\tau \vert m)$ and the ground-truth $r(s,a,\mu,m)$.
	\STATE Infer the context variable by $\hat{m} \sim q_\psi(m \vert \tau_E)$.
	\STATE Optimise a policy w.r.t. the reward $r(s,a,\mu,\hat{m})$.
	\STATE Evaluate the learned policy with $r(s,a,\mu,m)$.
\end{algorithmic}
\end{algorithm}

\section{Descriptions for Numerical Tasks}\label{app:task}

\subsection{Virus Infection} 

\subsubsection{Model.} This virus infection is used as a case study in \citep{cui2021approximately}. There is a large number of agents in a building. Each can choose between ``social distancing'' ($D$) or ``going out'' ($U$). If a ``susceptible'' ($S$) agent chooses social distancing, they may not become ``infected'' ($I$). Otherwise, an agent may become infected with a probability proportional to the number of infected agents. If infected, an agent will recover with a fixed chance every time step. Both social distancing and being infected have an associated negative reward.
Formally, let $\Smc = \{S,I\}, \Amc = \{U,D\}, r(s,a, \mu_t,m) = -\mathds{1}_{\{s = I\}} - m \cdot \mathds{1}_{\{s = D\}}$. The transition probability is given by
\begin{equation*}
	\begin{aligned}
		P(s_{t+1} = S \vert s_t = I, \cdot, \cdot) & = 0.3\\
		P(s_{t+1} = I \vert s_t = S, a_t = U, \mu_t) & = 0.9^2 \cdot \mu_t(I)\\
		P(s_{t+1} = I \vert s_t = S, a_t = D, \cdot) &= 0.
	\end{aligned}
\end{equation*}

\subsubsection{Settings.} The initial mean field $\mu_0$ is set as a uniform distribution, i.e, $\mu_0(s) = 1 / |\Smc|$ for all $s \in \Smc$. The context variable $m$ domain is $\{0.5, 1\}$, and the prior distribution $p(m)$ is a uniform distribution.

\subsection{Malware Spread}

\subsubsection{Model.}  The malware spread model is presented in \citep{huang2016mean,huang2017mean} and used as a simulated study for MFG in \citep{subramanian2019reinforcement}. This model is representative of several problems with positive externalities, such as flu vaccination and economic models involving the entry and exit of firms. 
Here, we present a discrete version of this problem:  Let $\Smc = \{0, 1, \ldots, 9\}$ denote the state space (level of infection), where $s = 0$ is the most healthy state and $s = 9$ is the least healthy state. The action space $\Amc = \{0, 1\}$, where $a = 0$ means $\mathtt{Do Nothing}$ and $a = 1$ means $\mathtt{Intervene}$. The dynamics is given by
\begin{equation*}
	s_{t+1} = \left\{
	\begin{aligned}
		& s_t + \lfloor \chi_t  ( 10- s_t ) \rfloor,  \text{~~if~~}  a_t = 0\\
		& 0, \text{~~if~~}a_t = 1
	\end{aligned}\right.,
\end{equation*}
where $\{\chi_t\}_{0\leq t \leq T}$ is a $[0, 1]$-valued i.i.d. process with probability density $f$. The above dynamics means the $\mathtt{Do Nothing}$ action makes the state deteriorate to a worse condition, while the $\mathtt{Intervene}$ action resets the state to the most healthy level. Rewards are coupled through the average health level of the population, i.e., $\langle \mu_{t} \rangle$ as defined in Eq.\eqref{eq:average_mf}. An agent incurs a cost $(k + \langle \mu_t \rangle) s_t$, which captures the risk of getting infected, and an additional cost of $\alpha$ for performing the $\mathtt{Intervene}$ action. The reward sums over all negative costs:
\begin{equation*}
	r(s_t, a_t, \mu_t,m) = -(m + \langle \mu_t \rangle)s_t/10 - \alpha \cdot a_t.
\end{equation*}

\subsubsection{Settings.} Following \citep{subramanian2019reinforcement}, we set $m = 0.2$ and $0.4$, $\alpha = 0.5$, and the probability density $f$ to the uniform distribution. The initial mean field $\mu_0$ and the prior distribution $p(m)$ are both set as a uniform distribution.

\subsection{Investment in Product Quality}

{\bf Model.} This model is adapted from \citep{weintraub2010computational} and \citep{subramanian2019reinforcement} that captures the investment decisions in a fragmented market with a large number of firms. Each firm produces the same kind of product. The state of a firm $s \in \Smc = \{ 0,1, \ldots, 9\}$ denotes the product quality. At each step, each firm decides whether or not to invest in improving the quality of the product. Thus the action space is $\Amc = \{0, 1\}$. When a firm decides to invest, its product quality increases uniformly at random from its current value to the maximum value 9 if the average quality in the market for that product is below a particular threshold $q$. If this average quality value is above $q$, then the product quality gets only half of the improvement compared to the former case. This implies that when the average quality in the economy is below $q$, it is easier for each firm to improve its quality. When a firm does not invest, its product quality remains unchanged. Formally, the dynamics is given by:

\begin{equation*}
	s_{t+1} = \left\{
	\begin{aligned}
		 & s_{t} + \lfloor \chi_t  ( 10- s_{t} ) \rfloor, \text{~~if~~}\langle \mu_{t} \rangle < q \text{~~and~~} a_{t} = 1\\
		 & s_{t} + \lfloor \chi_t  ( 10- s_{t} ) /2 \rfloor, \text{~~if~~}\langle \mu_{t} \rangle \geq q\text{~~and~~}a_{t} = 1\\
		 & s_{t}, \text{~~if~~}a_{t} = 0
	\end{aligned}\right..
\end{equation*} 

An agent incurs a cost due to its investment and earns a positive reward due to its own product quality and a negative reward due to the average product quality, which we denote by 
\begin{equation}\label{eq:average_mf}
\langle \mu_t \rangle \triangleq \sum_{s \in \Smc} s \cdot \mu_t(s).    
\end{equation}
The final reward is given as follows:
\begin{equation*}
	r(s_t,a_t,\mu_t,m) =  d \cdot s_{t} / 10 - c \cdot \langle \mu_{t} \rangle - m \cdot a_{t} 
\end{equation*}

\subsubsection{Settings.}  We set $d=0.3$, $c=0.2$ and probability density $f$ for $\chi_t$ as $U(0,1)$. We set the threshold $q$ to $4$. The context variable $m$ domain is $\{0.2, 0.5\}$ and the prior distribution $p(m)$ is a uniform distribution.

\section{Detailed Experimental Settings}\label{app:exp}

\subsection{Feature representations} We use one-hot encoding to represent states and actions. Let $\{1,2,\ldots, |\Smc|\}$ denote an enumeration of $\Smc$ and $\left[s_{[1]}, s_{[2]}, \ldots, s_{[|\Smc|]}\right]$  denote a vector of length $|\Smc|$, where each component stands for a state in $\Smc$. The state $j$ is denoted by $\big[0, \ldots, 0, s_{[j]} = 1, 0,$ $\ldots, 0 \big]$. An action is represented in the same manner. A mean field $\mu$ is represented by a vector $\left[\mu(s_{[1]}), \mu(s_{[2]}), \ldots, \mu(s_{[|\Smc|]})\right]$, where $\mu(s_{[i]})$ denotes the proportion of agents that are in the $i$th state. 

\subsection{Reward Models and Adaptive Samplers} The reward mode $r_\omega$ takes as input the concatenation of feature vectors of $s$, $a$, $\mu$ and $m$ and outputs a scalar as the reward. We adopt the neural network (a four-layer perceptron) with the Adam optimiser and the Leaky ReLU activation function. The sizes of the two hidden layers are both 64. The learning rate is $10^{-4}$. The adaptive sampler $\pi^{\theta_t}$ ($0 \leq t < T$) takes as input the feature vector of a state and outputs a distribution over the action set. The neural network architecture for each adaptive sampler is a five-layer perceptron with the Adam optimiser and the Leaky ReLU activation function. The sizes of the first two hidden layers are 64, the size of the third hidden layer is identical to the size of the action set, and the last layer is a softmax layer for generating a distribution over the action set.

\subsection{Computation of ERMFNE in Simulated Tasks} 
In ERMFNE expert training, we repeat the fixed point iteration to compute the MF flow. We terminate at the $i$th iteration if the mean squared error over all steps, and all states are below or equal to $10^{-10}$, i.e., $$\frac{1}{(T-1)|\Smc|} \sum_{t=1}^{T-1}\sum_{s \in \Smc} \left(\mu^{(i)}_t(s) - \mu^{(i-1)}_t(s)\right)^2 \leq 10^{-10}.$$

\subsection{Preprocessing the New York Yellow Taxi Dataset}

Most of the data in the dataset is collected by instruments, but it contains a considerable amount of dirty data that deviates significantly from standard cases due to various factors. The dataset also explicitly states that it does not guarantee the accuracy of each data entry. Therefore, cleaning and removing this dirty data is essential in the data preprocessing stage. Specifically, data where the passenger's drop-off time is earlier than the pick-up time and data with a total trip time of less than 1 minute is eliminated. Additionally, data with extreme outliers in latitude and longitude are removed. For example, while most data points are located in New York City, some individual data points may be found in the southern and eastern hemispheres or far beyond reasonable latitude and longitude values. These outlying location points should be considered dirty data and removed.

Each node (probably referring to different data points or sources) needs sufficient data, which is crucial for this model. If a node has too little data, some parameters in the model will be difficult to estimate reliably, significantly impacting the model's effectiveness. Therefore, a preliminary dataset analysis has been conducted, as shown in Tab.~\ref{tab:data_analysis}.

\begin{table}[!htp]
	\centering
	\begin{tabular}{c c c c c c}
		\toprule
		item & distance (mile) & pick-up longitude & pick-up latitude & drop-off longitude & drop-off latitude\\
		\midrule
		mean & 0.72 & -72.67 & 40.04 & -72.69 & 40.05\\
		min  & 0.00 & -770.4 & -0.07 & -735.6 & -0.07\\
		lower quartile & 0.50&-73.990& 40.740 & -73.990 & 40.740\\
		median &0.70&-73.980& 40.756 & -73.979& 40.756\\
		upper quartile & 0.90 & -73.966 & 40.770 & -73.965 & 40.772\\
		max & 41.9 & +0.077 & 404.6 & +0.076 & 456.5\\
		\bottomrule
	\end{tabular}
	\caption{Priliminary results of data analysis.}\label{tab:data_analysis}
\end{table}

Based on the preliminary analysis results shown in Tab.~\ref{tab:data_analysis} and further analysis of the dataset, for the subsequent experiments, data within the longitude range of -73.9$^\circ$ to -73.8$^\circ$ and latitude range of 40.7$^\circ$ to 40.8$^\circ$ will be selected. Most of the data in the dataset falls within this range, ensuring that each segmented node has adequate data.

Using a granularity of 0.01$^\circ$ for both latitude and longitude, the above region will be divided into 10 $×$ 10 grid nodes. In the data preprocessing stage, the latitude and longitude of the passenger's pick-up and drop-off locations will be converted into corresponding node numbers. Then, the passenger data within each node will be analysed, and statistics will be generated, such as determining the number of passengers and their destination distribution within each node.

Based on the model established in \citep{ata2019spatial}, it is necessary to calculate the destination node distribution for passengers boarding at each node. In this context, the overall distribution of passenger destinations is not important for the model; the model is concerned with calculating the destination node distribution separately for each departure node. Here are two examples:

\begin{figure}[htp]
	\centering
	\includegraphics[width=.4\textwidth]{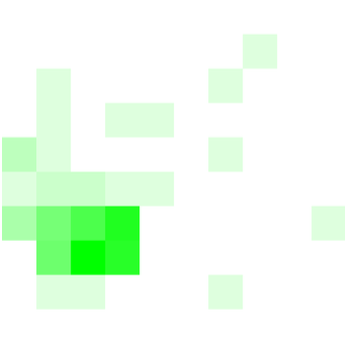}\hfil
	\includegraphics[width=.4\textwidth]{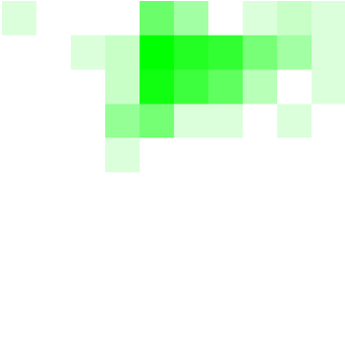}
	\caption{Passenger destination maps for Node 27 (left) and 52 (right).}\label{fig:27and52}
\end{figure}

Example 1: Passengers boarding from the node in the bottom-left region (node number 27) are the most frequent. The destination node distribution for these passengers is illustrated in Fig.~\ref{fig:27and52}. The darker green indicates a higher number of passengers, lighter green represents fewer passengers, and white areas indicate no passengers meeting this criterion. The statistical results show that 41.10\% of passengers have their destination at the same node they depart from (node number 27), 20.83\% have their destination towards the upper-right direction from the departure node (node number 36), and 17.61\% have their destination towards the right side of the departure node (node number 37)—passengers with destinations within one ring adjacent to the departure node account for a combined 95.27\%. Additionally, a small number of passengers travel to farther areas within the city, covering almost the entire selected latitude and longitude range.

Example 2: The second most common departure node is from the node in the upper-middle region (node number 52). The destination node distribution for passengers boarding from this node is shown in Fig.~\ref{fig:27and52}. Similar to the previous example, darker green represents a higher number of passengers. The statistical results indicate that only 8.37\% of passengers boarding from this node have their destination at the same node (node number 52). The majority of passengers are heading towards the upper-left direction from the departure node (node number 41, accounting for 27.63\%) and the left side of the departure node (node number 42, accounting for 20.23\%)—passengers with destinations within one ring adjacent to the departure node account for a combined 87.55\%. There are also a few passengers going to other areas. Still, unlike the previous example, all passengers boarding from this node are heading to the upper region, and no passengers are going to the lower region. These two examples illustrate significant differences in the destination distribution of passengers boarding from different locations. The destination distribution has been calculated for all departure nodes.

Additionally, in the data preprocessing stage, it is necessary to calculate the distribution of taxi positions at the starting moment. Due to the lack of unique identifiers for each taxi in the dataset, it is impossible to differentiate which passenger-carrying data belongs to the same taxi. Therefore, to determine the starting positions of taxis for a day, it is considered to use the pick-up locations from the first few epochs within the simulated time period. The resulting distribution is shown in Fig.~\ref{fig:setup}, where blue represents fewer taxis starting the day from that position, green indicates a moderate number, and red indicates more taxis starting the day from that position. The figure shows that a significant number of taxi drivers start their day's operation from the central-left area. Additionally, there are some taxis starting from the left and upper-middle areas, while fewer taxi drivers start their day from other areas.

The above diagrams show various instances of mismatch: areas with high passenger demand for rides, areas with high one-way taxi earnings, and areas with abundant taxi supply, and these areas do not overlap. Regions with a high number of passengers may need more taxi supply, leading to inadequate fulfilment of passenger ride demands. Conversely, regions with a high number of taxis may need more passenger demand, resulting in revenue losses for taxis during idle periods. Moreover, areas with high passenger ride demand might yield low operational profits. These instances of mismatch could be more favourable for both drivers and passengers.

\subsection{Hardware}

The hardware and operating system of the experimental environment are shown in Tab.~\ref{tab:hardware}.

\begin{table}[!htp]
	\centering
	\begin{tabular}{c c}
		\toprule
		item & Model and version\\
		\midrule
		CPU & Intel Xeon CPU E5-2620 v2 @ 2.10 GHz\\
		GPU  & Nvidia Tesla P-100\\
		Memory & DDR4 32G\\
		Operating System & Ubuntu 16.04\\
		\bottomrule
	\end{tabular}
	\caption{The hardware and operating system of the experimental environment.}\label{tab:hardware}
\end{table}

%% file: anonymous-submission-latex-2024.bbl
\begin{thebibliography}{43}
\providecommand{\natexlab}[1]{#1}

\bibitem[{Andrychowicz et~al.(2016)Andrychowicz, Denil, Gomez, Hoffman, Pfau,
  Schaul, Shillingford, and De~Freitas}]{andrychowicz2016learning}
Andrychowicz, M.; Denil, M.; Gomez, S.; Hoffman, M.~W.; Pfau, D.; Schaul, T.;
  Shillingford, B.; and De~Freitas, N. 2016.
\newblock Learning to learn by gradient descent by gradient descent.
\newblock \emph{Advances in neural information processing systems}, 29.

\bibitem[{Ata, Barjesteh, and Kumar(2019)}]{ata2019spatial}
Ata, B.; Barjesteh, N.; and Kumar, S. 2019.
\newblock Spatial pricing: An empirical analysis of taxi rides in New York
  City.
\newblock In \emph{The University of Chicago Booth School of Business Chicago,
  IL Working paper}.

\bibitem[{Cardaliaguet and Hadikhanloo(2017)}]{cardaliaguet2017learning}
Cardaliaguet, P.; and Hadikhanloo, S. 2017.
\newblock Learning in mean field games: the fictitious play.
\newblock \emph{ESAIM: Control, Optimisation and Calculus of Variations},
  23(2): 569--591.

\bibitem[{Carmona, Delarue, and Lachapelle(2013)}]{carmona2013control}
Carmona, R.; Delarue, F.; and Lachapelle, A. 2013.
\newblock Control of McKean--Vlasov dynamics versus mean field games.
\newblock \emph{Mathematics and Financial Economics}, 7(2): 131--166.

\bibitem[{Chen et~al.(2022)Chen, Zhang, Liu, and Hu}]{chen2022individual}
Chen, Y.; Zhang, L.; Liu, J.; and Hu, S. 2022.
\newblock Individual-level inverse reinforcement learning for mean field games.
\newblock In \emph{Proceedings of the 21st International Conference on
  Autonomous Agents and Multi-agent Systems}.

\bibitem[{Chen et~al.(2023)Chen, Zhang, Liu, and
  Witbrock}]{chen2023adversarial}
Chen, Y.; Zhang, L.; Liu, J.; and Witbrock, M. 2023.
\newblock adversarial inverse reinforcement learning for mean field games.
\newblock In \emph{Proceedings of the 22nd International Conference on
  Autonomous Agents and Multi-agent Systems}.

\bibitem[{Cui and Koeppl(2021)}]{cui2021approximately}
Cui, K.; and Koeppl, H. 2021.
\newblock Approximately Solving Mean Field Games via Entropy-Regularized Deep
  Reinforcement Learning.
\newblock In \emph{International Conference on Artificial Intelligence and
  Statistics}, 1909--1917. PMLR.

\bibitem[{Duan et~al.(2016)Duan, Schulman, Chen, Bartlett, Sutskever, and
  Abbeel}]{duan2016rl}
Duan, Y.; Schulman, J.; Chen, X.; Bartlett, P.~L.; Sutskever, I.; and Abbeel,
  P. 2016.
\newblock Rl$^2$: Fast reinforcement learning via slow reinforcement learning.
\newblock \emph{arXiv preprint arXiv:1611.02779}.

\bibitem[{Elie et~al.(2020)Elie, P{\'e}rolat, Lauri{\`e}re, Geist, and
  Pietquin}]{elie2020convergence}
Elie, R.; P{\'e}rolat, J.; Lauri{\`e}re, M.; Geist, M.; and Pietquin, O. 2020.
\newblock On the Convergence of Model Free Learning in Mean Field Games.
\newblock In \emph{Thirty-Fourth AAAI Conference on Artificial Intelligence},
  7143--7150.

\bibitem[{Finn, Abbeel, and Levine(2017)}]{finn2017model}
Finn, C.; Abbeel, P.; and Levine, S. 2017.
\newblock Model-agnostic meta-learning for fast adaptation of deep networks.
\newblock In \emph{International conference on machine learning}, 1126--1135.
  PMLR.

\bibitem[{Fu, Luo, and Levine(2018)}]{fu2018learning}
Fu, J.; Luo, K.; and Levine, S. 2018.
\newblock Learning Robust Rewards with Adverserial Inverse Reinforcement
  Learning.
\newblock In \emph{International Conference on Learning Representations}.

\bibitem[{Fu et~al.(2021)Fu, Tacchetti, Perolat, and
  Bachrach}]{fu2021evaluating}
Fu, J.; Tacchetti, A.; Perolat, J.; and Bachrach, Y. 2021.
\newblock Evaluating strategic structures in multi-agent inverse reinforcement
  learning.
\newblock \emph{Journal of Artificial Intelligence Research}, 71: 925--951.

\bibitem[{Ganapathi~Subramanian et~al.(2020)Ganapathi~Subramanian, Poupart,
  Taylor, and Hegde}]{ganapathi2020multi}
Ganapathi~Subramanian, S.; Poupart, P.; Taylor, M.~E.; and Hegde, N. 2020.
\newblock Multi Type Mean Field Reinforcement Learning.
\newblock In \emph{Proceedings of the 19th International Conference on
  Autonomous Agents and MultiAgent Systems}, 411--419.

\bibitem[{Ghosh and Aggarwal(2020)}]{ghosh2020model}
Ghosh, A.; and Aggarwal, V. 2020.
\newblock Model free reinforcement learning algorithm for stationary mean field
  equilibrium for multiple types of agents.
\newblock \emph{arXiv preprint arXiv:2012.15377}.

\bibitem[{Gomes, Mohr, and Souza(2010)}]{gomes2010discrete}
Gomes, D.~A.; Mohr, J.; and Souza, R.~R. 2010.
\newblock Discrete time, finite state space mean field games.
\newblock \emph{Journal de Math{\'e}matiques Pures et Appliqu{\'e}es}, 93(3):
  308--328.

\bibitem[{Goodfellow et~al.(2014)Goodfellow, Pouget-Abadie, Mirza, Xu,
  Warde-Farley, Ozair, Courville, and Bengio}]{goodfellow2014generative}
Goodfellow, I.; Pouget-Abadie, J.; Mirza, M.; Xu, B.; Warde-Farley, D.; Ozair,
  S.; Courville, A.; and Bengio, Y. 2014.
\newblock Generative adversarial nets.
\newblock \emph{Advances in neural information processing systems}, 27.

\bibitem[{Guo et~al.(2019)Guo, Hu, Xu, and Zhang}]{guo2019learning}
Guo, X.; Hu, A.; Xu, R.; and Zhang, J. 2019.
\newblock Learning mean-field games.
\newblock In \emph{Advances in Neural Information Processing Systems},
  4967--4977.

\bibitem[{Huang and Ma(2016)}]{huang2016mean}
Huang, M.; and Ma, Y. 2016.
\newblock Mean field stochastic games: Monotone costs and threshold policies.
\newblock In \emph{2016 IEEE 55th Conference on Decision and Control (CDC)},
  7105--7110. IEEE.

\bibitem[{Huang and Ma(2017)}]{huang2017mean}
Huang, M.; and Ma, Y. 2017.
\newblock Mean field stochastic games with binary actions: Stationary threshold
  policies.
\newblock In \emph{2017 IEEE 56th Annual Conference on Decision and Control
  (CDC)}, 27--32. IEEE.

\bibitem[{Huang et~al.(2006)Huang, Malham{\'e}, Caines et~al.}]{huang2006large}
Huang, M.; Malham{\'e}, R.~P.; Caines, P.~E.; et~al. 2006.
\newblock Large population stochastic dynamic games: closed-loop McKean-Vlasov
  systems and the Nash certainty equivalence principle.
\newblock \emph{Communications in Information \& Systems}, 6(3): 221--252.

\bibitem[{Lasry and Lions(2007)}]{lasry2007mean}
Lasry, J.-M.; and Lions, P.-L. 2007.
\newblock Mean field games.
\newblock \emph{Japanese Journal of Mathematics}, 2(1): 229--260.

\bibitem[{Mishra et~al.(2017)Mishra, Rohaninejad, Chen, and
  Abbeel}]{mishra2017meta}
Mishra, N.; Rohaninejad, M.; Chen, X.; and Abbeel, P. 2017.
\newblock Meta-learning with temporal convolutions.
\newblock \emph{arXiv preprint arXiv:1707.03141}, 2(7): 23.

\bibitem[{Ng, Harada, and Russell(1999)}]{ng1999policy}
Ng, A.~Y.; Harada, D.; and Russell, S. 1999.
\newblock Policy invariance under reward transformations: Theory and
  application to reward shaping.
\newblock In \emph{ICML}, volume~99, 278--287.

\bibitem[{Ng and Russell(2000)}]{ng2000algorithms}
Ng, A.~Y.; and Russell, S.~J. 2000.
\newblock Algorithms for Inverse Reinforcement Learning.
\newblock In \emph{Proceedings of the Seventeenth International Conference on
  Machine Learning}, 663--670.

\bibitem[{Rakelly et~al.(2019)Rakelly, Zhou, Finn, Levine, and
  Quillen}]{rakelly2019efficient}
Rakelly, K.; Zhou, A.; Finn, C.; Levine, S.; and Quillen, D. 2019.
\newblock Efficient off-policy meta-reinforcement learning via probabilistic
  context variables.
\newblock In \emph{International conference on machine learning}, 5331--5340.
  PMLR.

\bibitem[{Ratliff, Bagnell, and Zinkevich(2006)}]{ratliff2006maximum}
Ratliff, N.~D.; Bagnell, J.~A.; and Zinkevich, M.~A. 2006.
\newblock Maximum margin planning.
\newblock In \emph{Proceedings of the 23rd International Conference on Machine
  Learning}, 729--736.

\bibitem[{Ravi and Larochelle(2016)}]{ravi2016optimization}
Ravi, S.; and Larochelle, H. 2016.
\newblock Optimization as a model for few-shot learning.
\newblock In \emph{International conference on learning representations}.

\bibitem[{Santoro et~al.(2016)Santoro, Bartunov, Botvinick, Wierstra, and
  Lillicrap}]{santoro2016meta}
Santoro, A.; Bartunov, S.; Botvinick, M.; Wierstra, D.; and Lillicrap, T. 2016.
\newblock Meta-learning with memory-augmented neural networks.
\newblock In \emph{International conference on machine learning}, 1842--1850.
  PMLR.

\bibitem[{Seyed~Ghasemipour, Gu, and Zemel(2019)}]{seyed2019smile}
Seyed~Ghasemipour, S.~K.; Gu, S.~S.; and Zemel, R. 2019.
\newblock Smile: Scalable meta inverse reinforcement learning through
  context-conditional policies.
\newblock \emph{Advances in Neural Information Processing Systems}, 32.

\bibitem[{Subramanian and Mahajan(2019)}]{subramanian2019reinforcement}
Subramanian, J.; and Mahajan, A. 2019.
\newblock Reinforcement learning in stationary mean-field games.
\newblock In \emph{Proceedings of the 18th International Conference on
  Autonomous Agents and Multi-agent Systems}, 251--259.

\bibitem[{Sun et~al.(2018)Sun, Chen, Shi, Hong, Fu, and
  Sidiropoulos}]{sun2018learning}
Sun, H.; Chen, X.; Shi, Q.; Hong, M.; Fu, X.; and Sidiropoulos, N.~D. 2018.
\newblock Learning to optimize: Training deep neural networks for interference
  management.
\newblock \emph{IEEE Transactions on Signal Processing}, 66(20): 5438--5453.

\bibitem[{Thrun and Pratt(2012)}]{thrun2012learning}
Thrun, S.; and Pratt, L. 2012.
\newblock \emph{Learning to learn}.
\newblock Springer Science \& Business Media.

\bibitem[{Wang et~al.(2016)Wang, Kurth-Nelson, Tirumala, Soyer, Leibo, Munos,
  Blundell, Kumaran, and Botvinick}]{wang2016learning}
Wang, J.~X.; Kurth-Nelson, Z.; Tirumala, D.; Soyer, H.; Leibo, J.~Z.; Munos,
  R.; Blundell, C.; Kumaran, D.; and Botvinick, M. 2016.
\newblock Learning to reinforcement learn.
\newblock \emph{arXiv preprint arXiv:1611.05763}.

\bibitem[{Weintraub, Benkard, and Van~Roy(2010)}]{weintraub2010computational}
Weintraub, G.~Y.; Benkard, C.~L.; and Van~Roy, B. 2010.
\newblock Computational methods for oblivious equilibrium.
\newblock \emph{Operations research}, 58(4-part-2): 1247--1265.

\bibitem[{Xu et~al.(2019)Xu, Ratner, Dragan, Levine, and Finn}]{xu2019learning}
Xu, K.; Ratner, E.; Dragan, A.; Levine, S.; and Finn, C. 2019.
\newblock Learning a prior over intent via meta-inverse reinforcement learning.
\newblock In \emph{International conference on machine learning}, 6952--6962.
  PMLR.

\bibitem[{Yang et~al.(2018{\natexlab{a}})Yang, Ye, Trivedi, Xu, and
  Zha}]{yang2018learning}
Yang, J.; Ye, X.; Trivedi, R.; Xu, H.; and Zha, H. 2018{\natexlab{a}}.
\newblock Learning Deep Mean Field Games for Modeling Large Population
  Behavior.
\newblock In \emph{International Conference on Learning Representations}.

\bibitem[{Yang et~al.(2018{\natexlab{b}})Yang, Luo, Li, Zhou, Zhang, and
  Wang}]{yang2018mean}
Yang, Y.; Luo, R.; Li, M.; Zhou, M.; Zhang, W.; and Wang, J.
  2018{\natexlab{b}}.
\newblock Mean Field Multi-Agent Reinforcement Learning.
\newblock In \emph{35th International Conference on Machine Learning},
  volume~80, 5571--5580. PMLR.

\bibitem[{You et~al.(2019)You, Lu, Filev, and Tsiotras}]{you2019advanced}
You, C.; Lu, J.; Filev, D.; and Tsiotras, P. 2019.
\newblock Advanced planning for autonomous vehicles using reinforcement
  learning and deep inverse reinforcement learning.
\newblock \emph{Robotics and Autonomous Systems}, 114: 1--18.

\bibitem[{Yu, Song, and Ermon(2019)}]{yu2019multi}
Yu, L.; Song, J.; and Ermon, S. 2019.
\newblock Multi-Agent Adversarial Inverse Reinforcement Learning.
\newblock In \emph{International Conference on Machine Learning}, 7194--7201.

\bibitem[{Yu et~al.(2019)Yu, Yu, Finn, and Ermon}]{yu2019meta}
Yu, L.; Yu, T.; Finn, C.; and Ermon, S. 2019.
\newblock Meta-inverse reinforcement learning with probabilistic context
  variables.
\newblock \emph{Advances in Neural Information Processing Systems}, 32.

\bibitem[{Zhao, Song, and Ermon(2018)}]{zhao2018information}
Zhao, S.; Song, J.; and Ermon, S. 2018.
\newblock The information autoencoding family: A lagrangian perspective on
  latent variable generative models.
\newblock \emph{arXiv preprint arXiv:1806.06514}.

\bibitem[{Ziebart, Bagnell, and Dey(2010)}]{ziebart2010modeling}
Ziebart, B.~D.; Bagnell, J.~A.; and Dey, A.~K. 2010.
\newblock Modeling interaction via the principle of maximum causal entropy.
\newblock In \emph{Proceedings of the 27th International Conference on Machine
  Learning}, 1255--1262.

\bibitem[{Ziebart et~al.(2008)Ziebart, Maas, Bagnell, and
  Dey}]{ziebart2008maximum}
Ziebart, B.~D.; Maas, A.; Bagnell, J.~A.; and Dey, A.~K. 2008.
\newblock Maximum entropy inverse reinforcement learning.
\newblock In \emph{Proceedings of the 23rd AAAI Conference on Artificial
  Intelligence}, 1433--1438.

\end{thebibliography}
